%% file: neurips_2019.tex
\documentclass{article}

\usepackage[preprint]{neurips_2019}

\usepackage[utf8]{inputenc} 
\usepackage[T1]{fontenc}    
\usepackage{hyperref}       
\usepackage{url}            
\usepackage{booktabs}       
\usepackage{nicefrac}       
\usepackage{microtype}      
\usepackage{amsmath}
\usepackage{amssymb}
\usepackage{amsfonts}
\usepackage{amsthm}
\usepackage{mathtools}
\usepackage{subcaption}
\usepackage{algorithm}
\usepackage{algorithmic}
\usepackage{hyperref}
\usepackage{color}
\usepackage{wrapfig}
\DeclareMathOperator*{\argmax}{arg\,max}

\newcommand{\IR}{\mathbb{R}}

\newcommand{\eg}{\emph{e.g.}}
\newcommand{\ie}{\emph{i.e.}}

\newtheorem{lemma}{Lemma}

\title{A Bayesian Choice Model\\for Eliminating Feedback Loops}

\author{
G\"{o}khan \c{C}apan
\\\texttt{gokhan.capan@boun.edu.tr}\And
 \.{I}lker G\"{u}ndo\u{g}du\\\texttt{ilker.gundogdu@boun.edu.tr}\And
 Ali Caner T\"{u}rkmen\\ \texttt{caner.turkmen@boun.edu.tr}\And
 \c{C}a\u{g}r\i\,  Sofuo\u{g}lu\\\texttt{cagri.sofuoglu@boun.edu.tr}\And
 Ali Taylan Cemgil\\ \texttt{taylan.cemgil@boun.edu.tr}\\
Department of Computer Engineering, Bo\u{g}azi\c{c}i University, Istanbul, Turkey
}

\begin{document}

\maketitle

\begin{abstract}
Self-reinforcing feedback loops in personalization systems are typically caused by users choosing from a limited set of alternatives presented systematically based on previous choices.
We propose a Bayesian choice model built on Luce axioms that explicitly accounts for users' limited exposure to alternatives.
Our model is fair---it does not impose negative bias towards unpresented alternatives, and practical---preference estimates are accurately inferred upon observing a small number of interactions.
It also allows efficient sampling, leading to a straightforward online presentation mechanism based on Thompson sampling. 
Our approach achieves low regret in learning to present upon exploration of only a small fraction of possible presentations.
The proposed structure can be reused as a building block in interactive systems, e.g., recommender systems, free of feedback loops.
\end{abstract}

\input{introduction}
\input{setup}
\input{model}
\input{pres_mechanism}

\input{related}
\input{interaction}
\input{discussion}

\bibliography{bibl}

\pagebreak

\appendix
\renewcommand{\thesection}{\Alph{section}}
\input{supplement_model}

\input{supplement_ts}
\input{supplement_smc}
\input{supplement_sim}



\end{document}

%% file: introduction.tex
\section{Introduction}
In most modern digital applications, personalized user experience is shaped interactively. 
Users choose from a subset of available alternatives---\eg, products, movies, and news articles---selected by a personalization system, \eg, a recommender system, in line with their {\em preferences}.
The system, in turn, estimates the user's preferences based on her previous {\em discrete choices} among systematically presented alternatives.
A key source of bias is ignoring the fact that a choice is made from a cherry-picked and limited subset of options, leading to a self-reinforcing feedback loop.
That is, options that were never presented are unfairly penalized.

Preference estimates obtained from such systems are biased \citep{liangcausal} and inconsistent \citep{interactionrecsys}. 
Moreover, the user's interest may degenerate over time due to systematic exposure, leading to {\em echo chambers} \citep{degenerate}. 
The interplay of the user's choice and the system's presentation in a feedback loop, reinforcing the system's own biased belief, results in the so-called {\em filter bubble}---an unintentional form of censorship with unexpected economic and societal impact \citep{pariserbubble}.

We pose two questions towards tackling the feedback loop.
First, how does one model user preferences accounting for the bias introduced by systematic and limited presentations, ensuring all alternatives are treated fairly? 
Second, how does a system, built on such a model and aware of its limitations, learn to present the best subset of alternatives? Our contribution in this work attempts to answer both questions.

\begin{itemize}
	\item We introduce a Bayesian {\em choice model}, the Dirichlet-Luce model, that accounts for limited exposure to alternatives and conforms to Luce's choice axiom \citep{choice}.
	We study our model's properties, as well as methods for estimation and full Bayesian inference. 
	Dirichlet-Luce posterior inference achieves pairwise {\em preference aggregation} upon collecting statistics for $O(K \log K)$ unique pairs.
	Most importantly, our model ensures {\em independence of unexplored alternatives}---marginal posterior probabilities of choosing options that were never presented are independent of other choices. 
	That is, the model is provably {\em fair} to options that were underrepresented in previous presentations, or newly added (\ie, cold-started) to the set of options.

	\item We propose a mechanism for {\em learning to present} the best subset of alternatives, casting the Dirichlet-Luce model as the central component of a bandit algorithm. 
	Particularly, we develop a sequential Monte Carlo algorithm to sample a presentation from the Dirichlet-Luce posterior.
	Our approach, an instance of {\em Thompson sampling} \citep{thompson33}, naturally composes presentations with items that were either frequently preferred or scarcely presented.
	Our algorithm achieves lower regret than state-of-the-art bandit algorithms in pairwise \citep{doublets} and $L$-wise \citep{toprank} preference scenarios. 
\end{itemize}

Overall, we provide a practical model, inference framework, and presentation mechanism that can deal with the inherent self-reinforcing feedback loop present in many interactive systems. 
Our model can be reused as a key building block for personalization and recommender systems.

We present the problem setup in Section~\ref{sec:setup} and introduce the Dirichlet-Luce model in Section~\ref{sec:dirichlet_luce}.
Based on the model, we introduce our bandit algorithm in Section~\ref{sec:presentation}, discuss related work in Section~\ref{sec:related}, and provide a set of empirical results on interaction simulations in Section~\ref{sec:simulations}. 
We conclude the paper in Section~\ref{sec:discussion}.

%% file: setup.tex
\section{Problem Setup}\label{sec:setup}
{\bf Choice Model}
A discrete choice model specifies the probability of a user choosing an option among $K$ discrete alternatives (or, {\em options}).
Suppose there is a total of $K$ options, but $K$ is large and the user can only be exposed to a limited number of options to choose from. 
That is, the set of {\em choices} $\{k_t\}_{t=1}^T$ are made from systematically selected subsets of all alternatives---{\em presentations} $\{C_t\}_{t=1}^{T}$. 
Here, $k_t \in [K]\coloneqq \{1, 2, \cdots, K\}$, and $C_t \in \mathcal{C}$, where $\mathcal{C}$ denotes the set of all non-empty subsets of $[K]$.
Naturally, $k_t \in C_t, \forall t$.

Here, we study the probability $p(k | C)$ of choosing an item $k$ from a presentation $C$.
We assume a vector of {\em preferences} $\theta_k$ that specifies the probability of choosing an option above all other options, \ie, $\theta_k = p(k | [K])$.
In the first part of our contribution, we study a Bayesian choice model where each choice is multinomial restricted to a presentation.
Our choice model conforms to the Luce choice axiom as it satisfies {\em independence of irrelevant alternatives}---choices are probabilistic, and the probability of choosing an option over another is independent of other items in (or, absent from) the presentation.
Given a presentation $C$, it assigns choice probability to an item $k \in C$ in proportion to other items $C$, as first described by \citet{bradley} for the case of pairwise preferences, and generalized by \citet{choice} and \citet{permutations} for presentations comprising $L>2$ options, \ie, $p(k|C) = \theta_k / \left(\sum_{\kappa \in C} \theta_\kappa \right)$.

{\bf Online Learning for Subset Selection}
Armed with a model of choice behavior, an obvious next step is to design a mechanism for selecting presentations (subsets) $C \subset [K]$. 
Here, the goal is to make a presentation of size $L<K$ with low regret, measured in terms of the attractiveness (to the user) of the options included in the presentation. 
Over the course of interactions, such a mechanism ``explores'' the alternatives to be included in a presentation and learns to make optimal presentations. 

Such sequential decision making problems are widely framed as {\em bandit problems} \citep{banditbook}. 
In contrast to the standard multi-armed bandits setup, our goal is to come up with a system that does not merely pick a ``basic'' arm, but a subset of dependent arms.
Furthermore, the learner does not receive feedback for all arms presented, but only observes a {\it winner} (preference over others) \citep{complexonline}. 
This bandit setup has been referred to as {\em dueling bandits} (for $L=2$) \citep{dueling}, and {\em battling bandits} ($L>2$) \citep{battle}.	

%% file: model.tex
\section{Dirichlet-Luce Model}\label{sec:dirichlet_luce}
Our first step is to write a choice model that fully specifies the probability of an option being chosen given a presentation $C$.
First, we take a {\em restricted multinomial} likelihood conditioned on the presentations $C_{1:T}$, and an underlying preference vector $\theta \in \Delta$ where  $\Delta$ denotes the $(K-1)$-probability simplex. 
Namely, we study the likelihood,
\begin{equation}\label{eq:likelihood2}
    p(k_{1:T}\mid \theta, C_{1:T}) 
        = \frac{\prod_k {\theta_{k}}^{y_k}}
                {\prod_{C \in \mathcal{C}} {(\sum_{\kappa\in C}\theta_{\kappa}})^{\mu(C)}},
\end{equation}
where $\mu$ and $y_k$ are statistics defined as
\begin{equation*}
\mu(C) = \sum_{t=1}^T\left[C_t = C\right], \qquad
y_k = \sum_{t=1}^T\left[k_t = k\right].
\end{equation*}
That is, $\mu(C)$ is the multiplicity of a presentation $C$, and $y_k$ the number of times an option $k$ is chosen. 
Note that these quantities are related, both are marginals of $\nu(k,C) = \sum_{t=1}^T \left[ C_t = C \right] \left[ k_t = k \right]$ that is defined to be the number of times option $k$ was chosen when the presentation was $C$: $\mu(C) = \sum_{k} \nu(k,C) $ and $y_k = \sum_{C \in \mathcal{C}} \nu(k,C)$.

We can model the cases of (i) opting not to choose by a dummy option that is always an element of any presentation, and (ii) making multiple choices by repeating the presentation as many times as a choice was made. Therefore, we will not deal with these cases explicitly.

The probability mass function of (\ref{eq:likelihood2}) admits the Dirichlet conjugate prior
 $p(\theta\mid \alpha) = \mathcal{D}(\theta; \alpha) = \Gamma\left(\sum_k \alpha_k\right) \prod_k \Gamma(\alpha_k)^{-1} \prod_k \theta_k^{\alpha_k - 1},$
where $\alpha = \alpha_{1:K}$ is the vector of concentration parameters. 
Upon closer inspection, we observe that a wider family of distributions satisfies conjugacy to the likelihood (\ref{eq:likelihood2}).
The likelihood suggests a generalized family of Dirichlet distributions,
\begin{equation}\label{eq:prior}
p(\theta \mid \alpha, \beta) \propto \prod_k \theta_k^{\alpha_k-1} \prod_{C\in \mathcal{C}} \left(\sum_{\kappa \in C}\theta_\kappa\right)^{-\beta(C)}.
\end{equation}
The parameters $\beta(C), \alpha_k$ denote the pseudo-counts of presentations and choices respectively.
For consistency, we also require that $\sum_C \beta(C) = \sum_k \alpha_k$, \ie, prior parameters lend themselves to interpretation as the marginals of a pseudo-contingency table.
It is easy to see, when $\beta = \beta_0$ where $\beta_0(C) = [C = [K]] \sum_{k} \alpha_k$,
$p(\theta \mid \alpha, \beta)$ reduces to the Dirichlet distribution as $p(\theta \mid \alpha, \beta_0) = \mathcal{D}(\theta; \alpha)$.

Completing our specification of the {\bf Dirichlet-Luce model}, the posterior distribution of preferences follow 
\begin{equation}\label{eq:dir_luce_posterior}
p(\theta\mid k_{1:T}, C_{1:T}, \alpha, \beta) \propto \prod_k \theta_k^{\alpha_k+y_k-1} \prod_{C\in \mathcal{C}}\left(\sum_{\kappa \in C} \theta_\kappa\right)^{-\mu(C) - \beta(C)}.
\end{equation}

One can adjust $\alpha$ to place prior belief on user preferences analogously to a Dirichlet prior, setting $\beta([K]) = \sum_k \alpha_k$. 
Prior beliefs on presentations can be embedded by adjusting $\beta$.
If $C_t = [K] \, \forall t$, the model reduces to the Dirichlet-Multinomial.

A special case of the posterior form (\ref{eq:dir_luce_posterior}), for pairwise preferences, was first introduced by \citet{bayesianbt73}. 
Although an inference procedure was not described, the normalizing constant was recognized in a series form later referred to as a ``very complicated function of factorials'' by \citet{bayesianbt77}. 
In fact, the complicated part of the normalizing constant is a special hypergeometric function, Carlson's $\mathcal{R}$ function \citep{rfunction}, and exact inference is intractable in general.
In the rest of this work, we rely on Monte Carlo methods for approximate posterior inference.

A general form of the proposed density was studied by \citet{dickeyhypergeometric} following Carlson's work.
\citet{dickeycensored} utilized this generalization for Bayesian analysis of multinomial cell probabilities under censored observations, a problem which can be thought of as the ``inverse'' of ours.
Our model also appears as a special case of the Hyperdirichlet distribution studied by \citet{hyperdirichlet}, where we additionally require consistency in parameters, \ie, $\sum_k \alpha_k + y_k = \sum_C \beta(C) + \mu(C)$.
We study the Dirichlet-Luce model's properties, estimation and inference methods in the supplementary material, taking a closer look at deterministic and Monte Carlo-based approximate inference. 

{\bf Preference Learning} 
Computing posterior predictive choice probabilities under the Dirichlet-Luce model requires approximate inference. 
In particular, the predicted preference probability of an option $k$ given presentation $C_{T+1}$ is equal to the expectation of preference ratios under the posterior
\begin{align*}\label{eq:pref2}
&p(k_{T+1}=k\mid C_{T+1},  k_{1:T}, C_{1:T}, \alpha, \beta) =\mathbb{E}_{p(\theta \mid k_{1:T}, C_{1:T})} \left[\frac{\theta_{k}}{\sum_{\kappa\in C_{T+1}} \theta_\kappa}\right].
\end{align*}

Under mild conditions, the posterior is log-concave with respect to log-preferences.
The gradient and Hessian of the posterior can be computed in time linear in the number of observations, making gradient-based maximum a posteriori (MAP) estimation scalable and easy to implement.
Furthermore, MAP estimates converge to the true latent preferences fairly quickly, as demonstrated in the supplementary material.

Upon first inspection of the posterior (\ref{eq:dir_luce_posterior}), one of the first concerns is the dimensionality of the sufficient statistic $\mu$. 
The model informs its knowledge of the preferences $\theta$ via statistics collected on a set that scales combinatorially.
This raises a key question on the size of the sample required for an accurate preference estimate.

The key is to note that taking a single preference vector leads to the implicit assumption that preferences are (stochastically) transitive---options admit a total ordering in their probability of being chosen against all others.
By analogy, a sorting algorithm---relying on transitivity---would find the ordering of $K$ options with $O(K\log K)$ pairwise comparisons.
Similarly, in the field of active learning, stochastic ranking from pairwise preferences has been widely explored \citep{maxingrankingactive, duelingranking, rankingactive, busapreference, busatopk, mohajeractivetopk}.
Here too, the object is to attain a probably approximately correct ranking of preferences within low sample complexity. 
Analogously to our case, the algorithm receives a stochastic comparison (random preference) feedback instead of a deterministic comparison result. 

\begin{wrapfigure}{r}{0.55\textwidth}
\centering
\vspace{0pt}
\includegraphics[width=\linewidth]{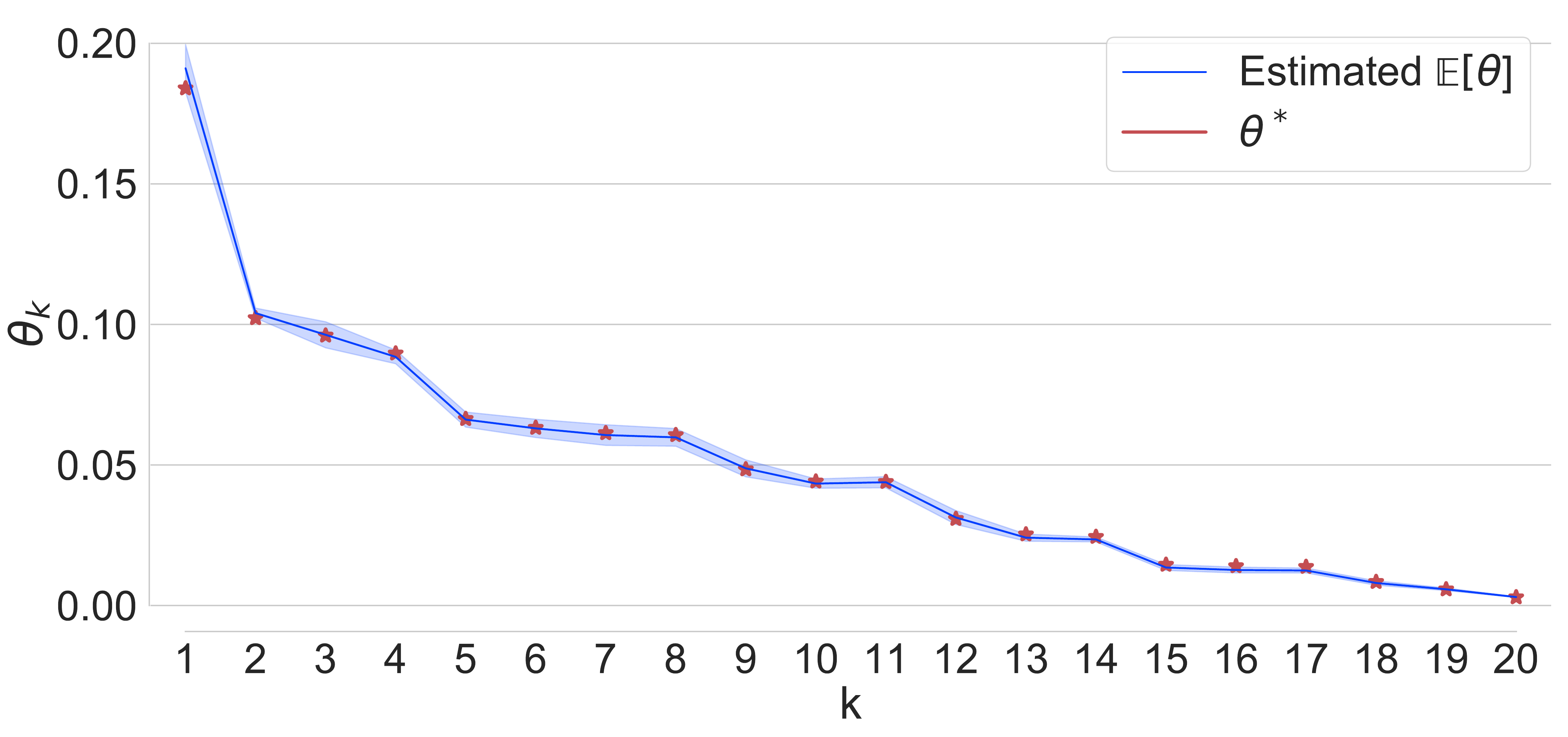}
\caption{Estimated (from 1000 samples) posterior mean (over 50 runs) conditioned on an actively selected data set (by Merge-Rank with bias and confidence parameters $\epsilon=0.05$ and $\delta=0.1$) of presentations due to transitivity and with simulated choices. $\theta^*$ is ordered, and $\mathbb{E}[\theta]$ estimates are conformably permuted for visualization. Shaded region denotes the standard deviation.}
\vspace{10pt}
\label{fig:ranking1}
\end{wrapfigure}

In this light, we explore whether the Dirichlet-Luce construction is able to recover a good representation of preferences, given a number of samples on the same order as a stochastic ranking algorithm.
We take the \textit{Merge-Rank} algorithm \citep{maxingrankingactive}, a stochastic variant of the {\it mergesort} algorithm for active subset selection---selecting pairs of options to ask for a preference feedback from the environment.
We fix a preference vector $\theta^*$ and run the Merge-Rank algorithm, generating a set of pairwise comparisons $C_{1:T}$ (corresponding to our presentations) and stochastic feedback $k_{1:T}$ (corresponding to choices).
In Figure~\ref{fig:ranking1}, we find that $\theta^*$ can be recovered accurately based on the same number of samples required by a stochastic ranker.
As in mergesort, the number of unique presentations is $O(K\log K)$.

{\bf Conflicting Choices} 
Our previous argument highlights the key ingredient that our model relies on to recover preferences---stochastic transitivity.
This observation raises a natural question, what if transitivity does not hold?
That is, how does the model treat conflicting (cyclical) choice behavior?

It is not hard to see that the likelihood is invariant for all realizations with identical $\mu$ and $y$. 
That is, the posterior ignores the associations between choices and particular presentations.
Conflicting choices are treated as {\it draws} by the posterior, as illustrated in Figure~\ref{fig:cyclic}.
These observations suggest a natural next step in utilizing our model. 
Cyclical choice behavior can be modeled as a mixture of preferences $\theta$.
Then, a mixture of Dirichlet-Luce densities can be used to capture different {\em modalities} of transitive preferences under limited subsets of alternatives, modeling complex choice behavior. There is also further evidence about the plausibility of such an approach in the psychology literature \citep{intransitive, transitivity}.
While we focus on modeling a single, unimodal preference behavior in this paper, this direction remains an exciting opportunity for further work.

\begin{figure}
\centering
\begin{subfigure}{0.49\textwidth}
	\includegraphics[width=\linewidth]{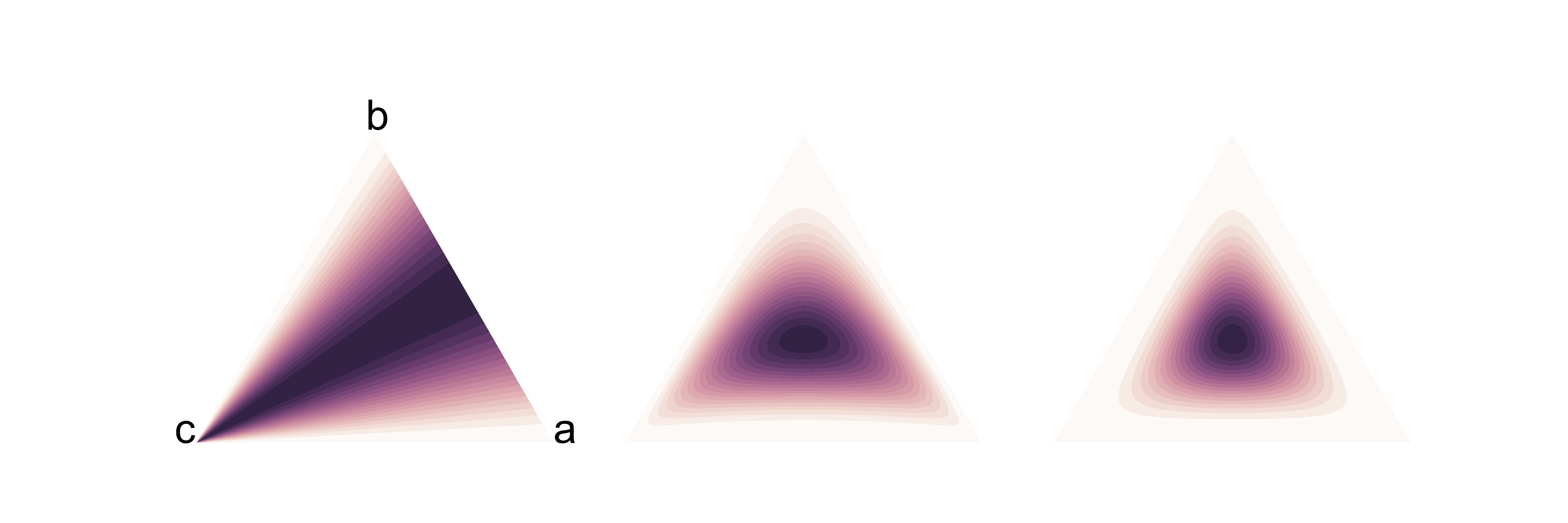}
	\caption{$(a<b)(a>b)(b<c)(b>c)(c<a)(c>a)$}
\end{subfigure}
\begin{subfigure}{0.49\textwidth}
	\includegraphics[width=\linewidth]{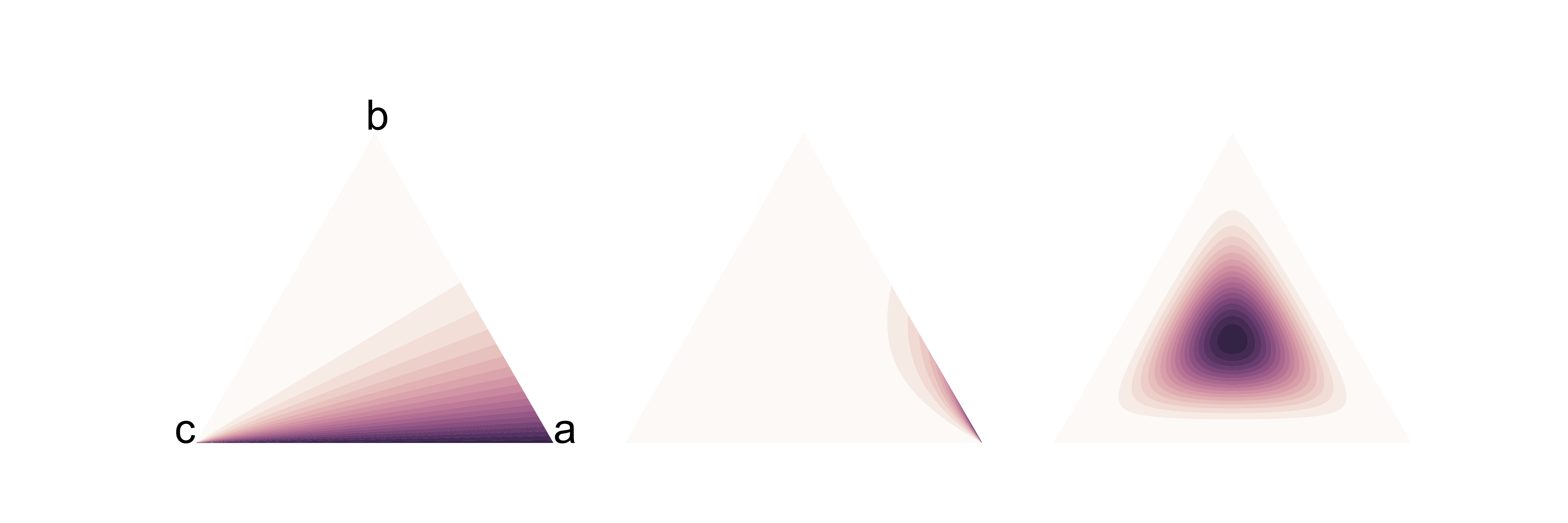}
	\caption{$(a>b)(a>b)(b>c)(b>c)(c>a)(c>a)$}
\end{subfigure}
\caption{Evolution of the exact posterior for $T=6$ observations. {\it Cyclic preferences} (b) are treated as {\it draws} (a).}
\label{fig:cyclic}
\end{figure}

{\bf Independence of Unexplored Options}
We now turn to the main benefit of our construction, its ability to keep invariant preference estimates of options that were never presented.
Particularly, the posterior density leads to ``fair'' preference estimates by keeping posterior marginals of $\theta_k$, where $k$ is an unexplored (never before presented) option, invariant independently of other choices. 
More formally, assume w.l.o.g. that option $k = 1$ is never presented, \ie, $\mu(C) = 0, \forall C \ni 1$.
It then follows, $p(\theta_1\mid \alpha, \beta_0, k_{1:T}, C_{1:T}) = p(\theta_1\mid \alpha, \beta_0)$.
This is in stark contrast to the Dirichlet-Multinomial model, in which choice observations impose negative bias on the marginals of all $\theta_k$, regardless of if they were ever presented. 
We give an illustration with $K=3$ options (Figure~\ref{fig:bubble}), and a proof in the supplementary material.
This result extends trivially to {\it groups} of options.

{\bf Fair Cold-Start}
Let us highlight one important implication of our previous remark. 
Under the Dirichlet-Luce posterior, alternatives that are newly introduced to the system benefit from the same invariance, or ``fairness,'' that other unpresented options enjoy. 
This property emerges as a natural way to ensure consistent inference of preferences for {\em cold-started} items. 

\begin{figure*}[t]
	\begin{subfigure}{0.33\textwidth}
		\includegraphics[width=\linewidth]{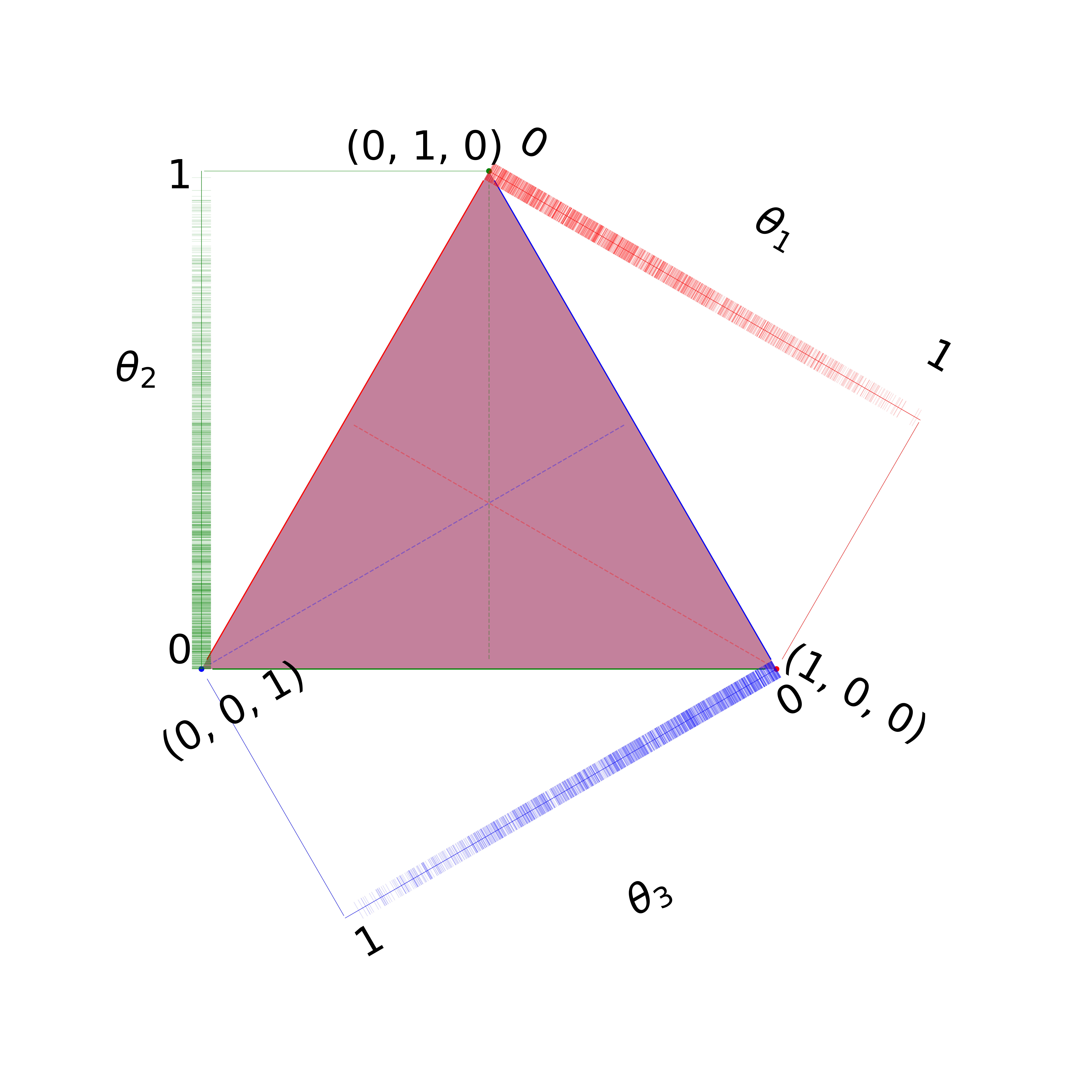}
		\caption{Prior}
	\end{subfigure}
	\begin{subfigure}{0.33\textwidth}
		\includegraphics[width=\linewidth]{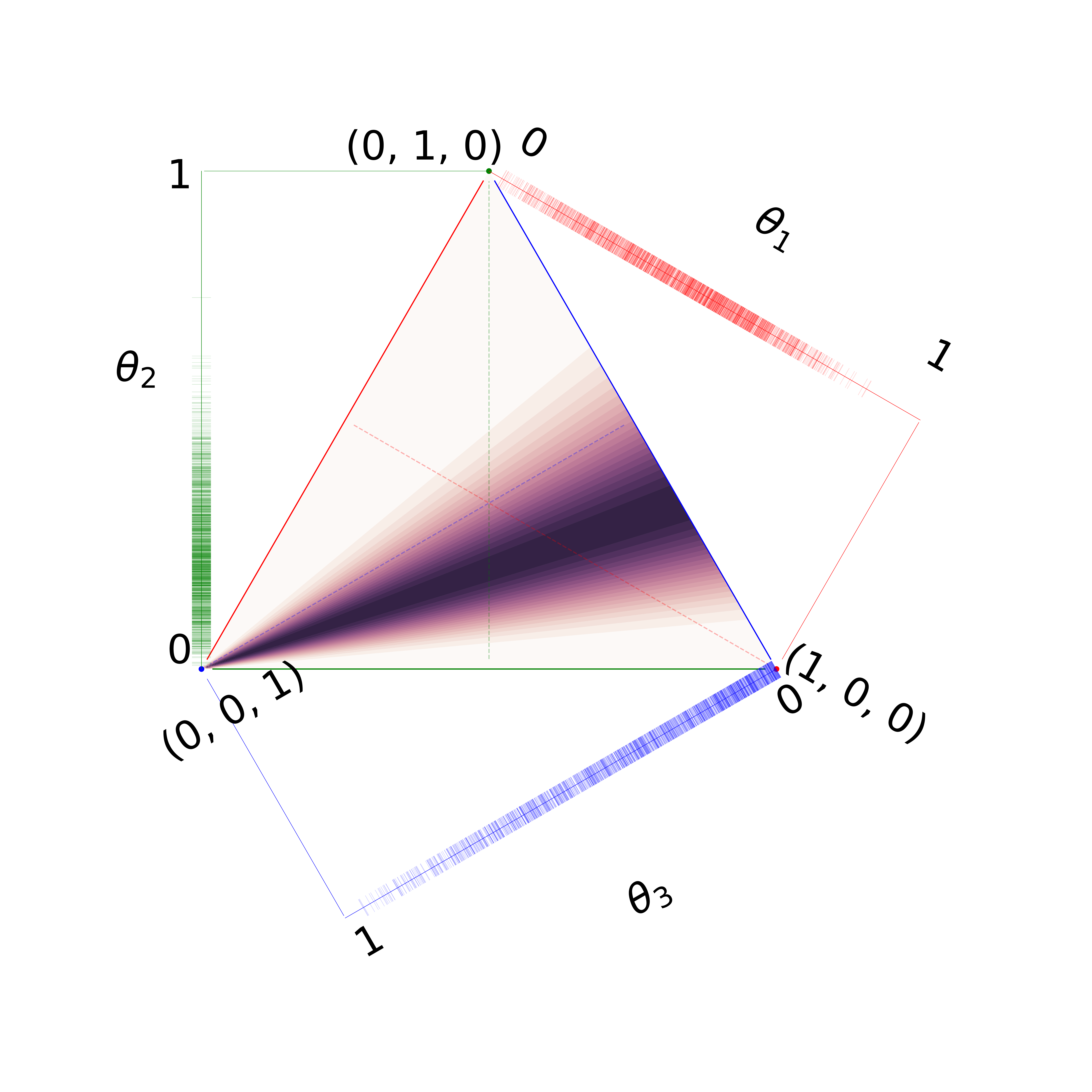}
		\caption{$p(\theta\mid k_{1:T}, C_{1:T}, \alpha, \beta_0)$}
	\end{subfigure}
	\begin{subfigure}{0.33\textwidth}
		\includegraphics[width=\linewidth]{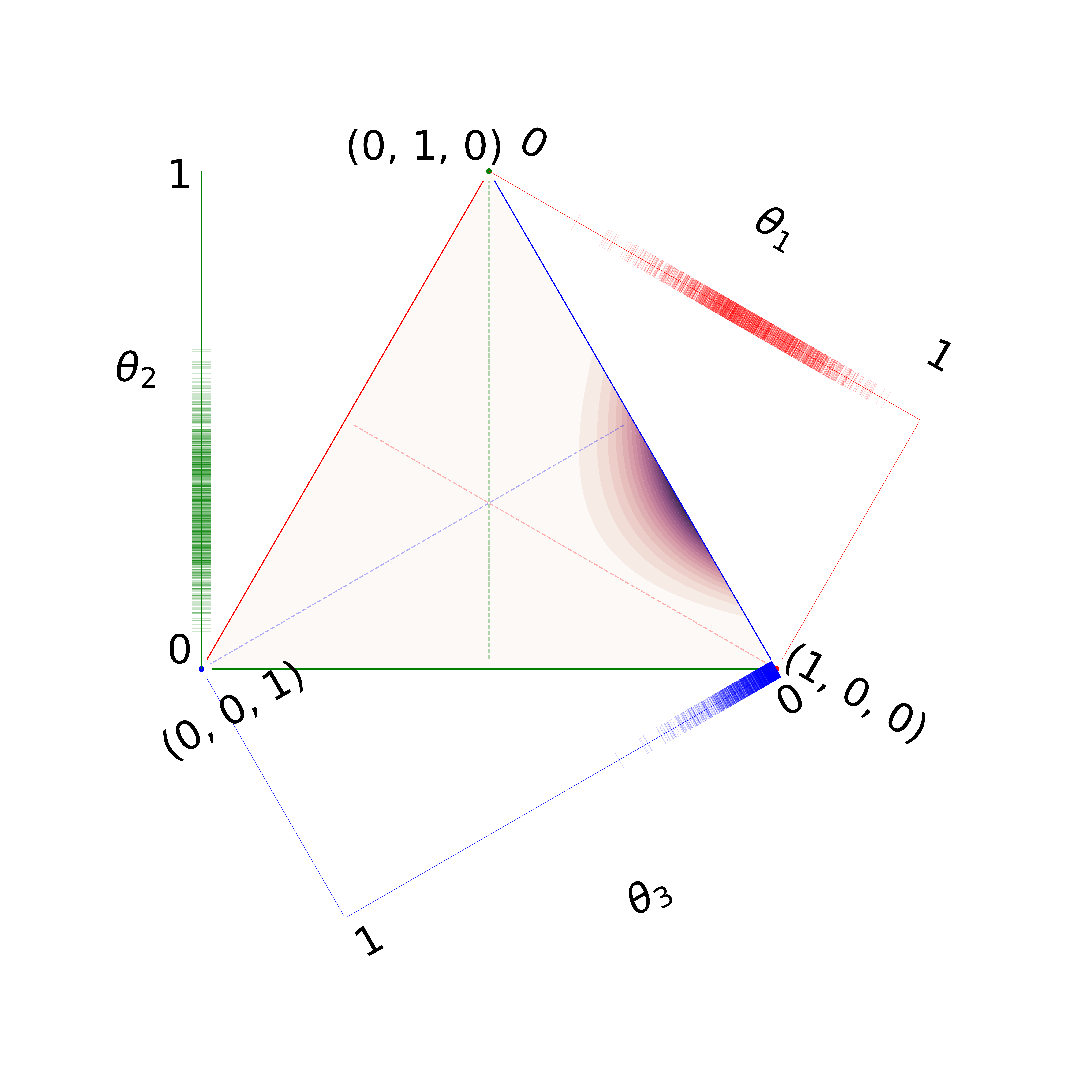}
		\caption{Dirichlet posterior $p(\theta\mid k_{1:T},\alpha)$, $C_{1:T}$ is ignored}
	\end{subfigure}
	\caption{Prior choice probabilities, along with the posterior where options $1$ and $2$ were preferred to presentation $\{1,2\}$ $10$ and $5$ times, respectively. Contours of the joint distribution are shown in a simplex plot. Samples from posterior marginals for each $k\in [K]$ are marked along an axis parallel to the line segment from the vertex where $\theta_k = 1$, perpendicular to the base (where $\theta_k = 0$). Since choices are made from $\{1,2\}$, marginal of $\theta_3$ must stay invariant as in (b). But if we ignore what presentations are made, the option $3$ is unfairly penalized, as shown in (c).}
	\label{fig:bubble}
\end{figure*}

%% file: pres_mechanism.tex
\section{Learning to Present}\label{sec:presentation}
Cast as the model assumption of an interactive system, the Dirichlet-Luce model provides fair and efficient preference estimates.
However, in a real-world scenario, the onus is on the system to select a subset of options to present.
That is, the system needs an efficient {\em presentation mechanism} that simultaneously explores the options which the user might like and exploits the current best alternatives.

Here, we frame this active preference learning scenario as a bandit problem.
In a bandit setting, the Bayesian construction of our model serves a dual purpose. 
First, new choice observations can be used to inform efficient approximate inference of latent preferences.
More importantly, however, posterior samples of the model serve as a natural means to manage the exploration-exploitation trade-off, inducing a presentation mechanism conditioned on previous choices ($k_{1:T}$) and the mechanism itself ($C_{1:T}$).

The posterior places high probability density on preferences where $\theta_k$ is high either when $k$ was chosen frequently or presented rarely.
Then, presenting top $L$ options of the vector $\theta$ sampled from the posterior would serve as a natural presentation mechanism.
This approach is an instance of Thompson sampling \citep{thompson33}, and affords fast exploration of the space of subsets for reasons analogous to the fast convergence of posterior preferences under Dirichlet-Luce model. Pseudo-code is given in Algorithm~\ref{alg:thompson}.

Thompson sampling plays an essential role in fair preference estimates together with the Dirichlet-Luce model.
Assume that options $\{1, 2\}$ were presented, and $1$ was chosen.
In the remaining rounds, if $1$ loses to other options in presentations where $2$ {\em does not} appear, the mean preference estimate of $\theta_2$ will {\em still decrease}---a side effect of the stochastic transitivity assumption.
However, Thompson sampling ensures options that are {\em underrepresented} in previous interactions will be presented, \ie, they will get a second chance.
In other words, while our model assumption reduces negative bias on options that do not appear in presentations, Thompson sampling ensures that they appear in subsequent presentations in order to form an accurate preference estimate. We study synthetic examples in the appendix.
\begin{wrapfigure}{r}{0.5\textwidth}
\vspace{+65pt}
\begin{minipage}{0.5\textwidth}
  \begin{algorithm}[H]
    \caption{Thompson Sampling for Presentation}
    \begin{algorithmic}\label{alg:thompson}
        \STATE {\bfseries Input:} $T$
        \STATE Initialize $\alpha$, and set $\beta\gets\beta_0$
        \STATE $\mu(C)\gets0$ for all $C\in\mathcal{C}$
        \STATE $y_k \gets 0$ for all $k\in [K]$
        \FOR{$t=1$ {\bfseries to} $T$}
        \STATE Sample $\theta \sim p_t(\theta\mid y, \mu, \alpha, \beta)$
        \STATE Form $C_t$ with top $L$ options from sampled $\theta$ 
        \STATE Get preference feedback $k_t$ to $C_t$
        \STATE $\mu (C_t) \gets \mu (C_t) + 1$ \COMMENT{Update sufficient statistics}
        \STATE $y_{k_t} \gets y_{k_t} + 1$
        \ENDFOR
    \end{algorithmic}
\end{algorithm}
\end{minipage}
\end{wrapfigure}

{\bf Sequential Monte Carlo Sampling}
We implement the sampling subroutine in Algorithm~\ref{alg:thompson} via a sequential Monte Carlo (SMC) algorithm \citep{chopin2002}.
Specifically, after each presentation, a set of importance-weighted particles are updated recursively based on the presentation-choice pair $(C_t, k_t)$.
When the {\em effective sample size} drops below a set threshold, we perform multinomial resampling followed by a Metropolis-within-Gibbs step for moving the particles, \ie, particles undergo a {\em resample-and-move} update.

It is worth noting that the conjugacy in our model leads to particle weight updates that take a negligible $O(L)$ time.
Posterior evaluations, which are expensive, are only needed during the resample-and-move step. 
However, as the number of observations $T$ grows, the posterior peaks around the latent preference vector and the algorithm requires resampling with decreasing frequency. 
That is, expensive resampling steps are dominant as the algorithm tends towards exploration, and very rare when the posterior is peaked and the algorithm ``commits'' a preference representation. 
The complete presentation mechanism with the SMC sampling procedure is described in the supplementary material.

%% file: related.tex
\section{Related Work}\label{sec:related}
Modeling implicit feedback data sets as conditionally multinomial observations is an approach that has received recent attention, see, for example, the works by \citet{youtube16deep}, \citet{liang2018variational}, and references therein. 
In line with the early works of \citet{bradley} for the pairwise preferences case, generalized by \citet{choice} and \citet{permutations} to $L$-wise preferences, a choice from a limited set can be assumed a restricted multinomial. 
Notably, \citet{collaborativecompetitive} also assume the same likelihood that we consider here, although they do not propose an inference method or presentation mechanism.

Random utilities underlying choice probabilities are commonly modeled by a collection of independent Gumbel-distributed random variables \citep{yellott77, rumsocialchoice2012}. 
A Bayesian treatment of this model appears in \citep{bayesianplackett2009}.
We model the dependence of choices (that were considered by the decision maker) via our generalized Dirichlet assumption.
As noted before, this is a general case of the formulation introduced by \citet{bayesianbt73}, and a special case of those described by \citet{dickeyhypergeometric} and \citet{hyperdirichlet}.

We compose presentations via the ranking induced by sampled preferences. 
For presentations comprising of $L>2$ options, we can think of this mechanism as an algorithm for online {\em learning to rank} \citep{learningtorank} with a document-based click (in our case, choice) model, ignoring the position of an option in the list of presented alternatives \citep[see][Section 3]{clickmodels}.
Alternative click models, which we do not consider here, were studied in the information retrieval literature \citep{positionbias}. 
We use the recently introduced ``TopRank'' algorithm \citep{toprank} as a baseline to measure the effectiveness of our presentation mechanism in online learning to rank.
TopRank subsumes various choice models including ours, and was shown to perform superior to previous work. 
For $L=2$, a closely related scenario is ``dueling bandits'' \citep{dueling}.
We use ``Double Thompson Sampling'' (DTS), introduced by \citet{doublets},  as a baseline in our pairwise preference simulations.
DTS has been referred to many times as one of the best performing dueling bandits algorithms \citep{doublets, multidueling, battle, duelingadvancements}. 
In {\em battling bandits}---the recently introduced generalization of the dueling bandits problem to subsets with $L>2$ elements--- \citet{battle} identified DTS as the best performing base method under various feedback models.

%% file: interaction.tex
\section{Experiments}\label{sec:simulations}
{\bf $L$-wise Presentations}
We first study an interactive system where a simulated user chooses from presentations of size $L=5$ based on a latent $\theta^*$, \ie, the probability of choosing an alternative $k$ is $p(k\mid C) = \frac{[k \in C]\theta^*_k}{\sum_{\kappa\in C} \theta^*_\kappa}$. 
We report a comparison of our method to the top-5 options chosen by {\em TopRank} \citep{toprank}.
We compare performances of the two algorithms by {\em cumulative regret} for the top-$N$ options included in the presentation by both algorithms,
\begin{equation*}
    R_T =  T \left(\max_{i_1, i_2, \cdots, i_N}\sum_{n = 1}^N \theta^*_{i_n}\right) - \sum_{t=1}^T \sum_{n=1}^N\theta^*_{\kappa_n^{(t)}}.
\end{equation*}
We use the ranking given by TopRank, and the ranking implied by the sampled preference vector in the case of Dirichlet-Luce.

We run experiments for $K=50$, and report cumulative regret. 
We run two separate experiments for sparse and dense $\theta^*$.
Across varying $N$, the Dirichlet-Luce bandit categorically achieves lower cumulative regret. 

\begin{figure}[h!]
\centering
\begin{subfigure}{0.42\textwidth}
	\includegraphics[width=\linewidth]{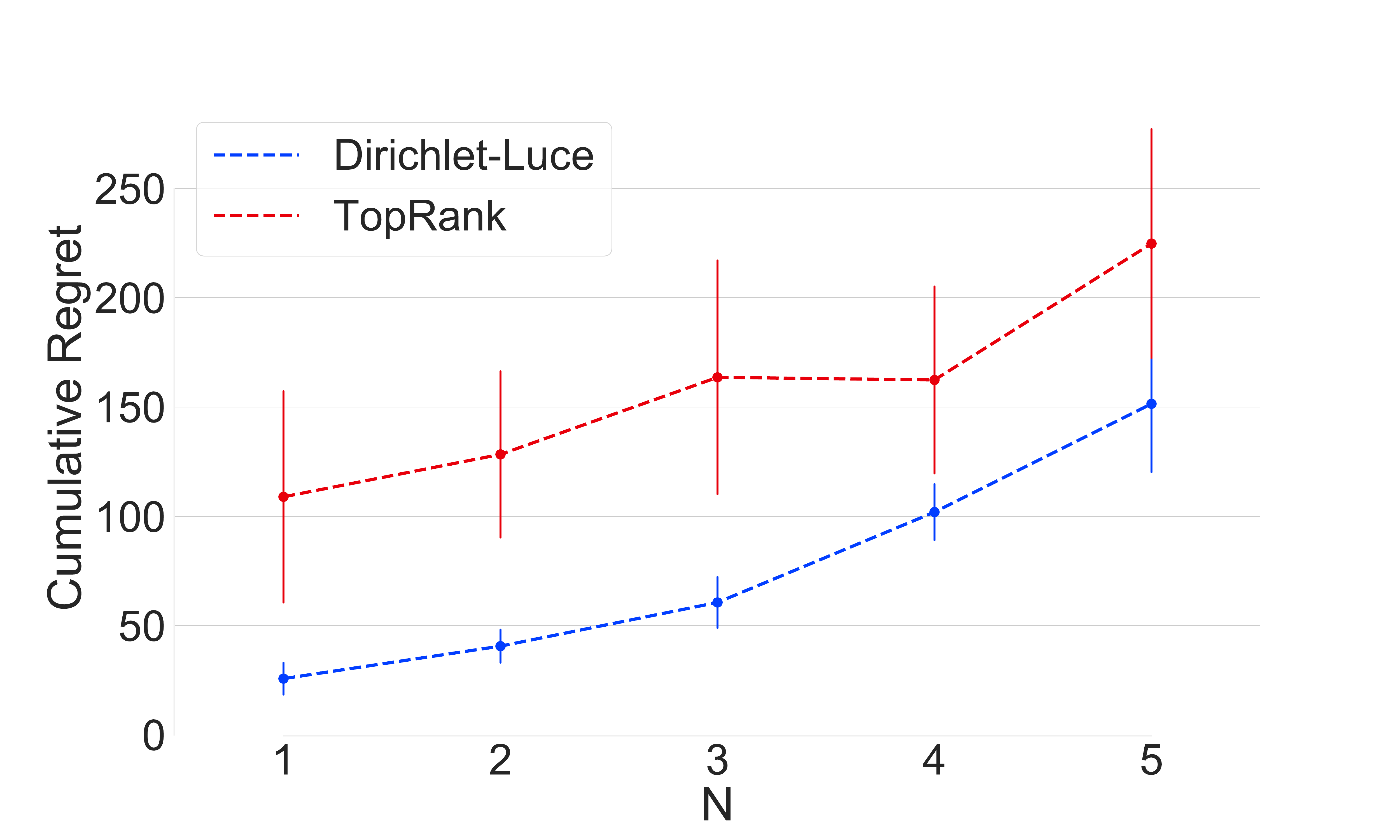}
	\caption{Sparse $\theta^*$}
\end{subfigure}
\begin{subfigure}{0.42\textwidth}
	\includegraphics[width=\linewidth]{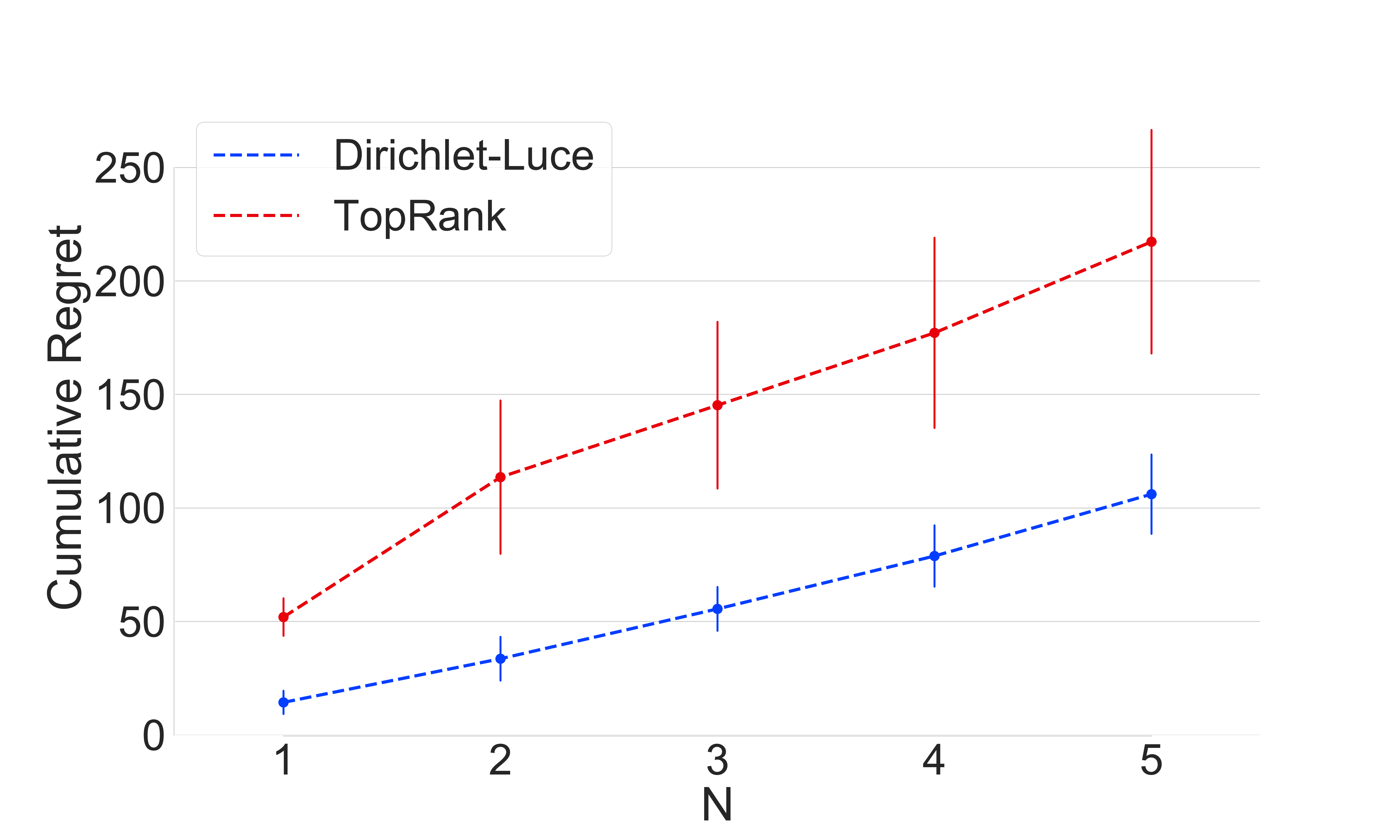}
	\caption{Dense $\theta^*$}
\end{subfigure}
\caption{Average cumulative regret at top-$N$ alternatives included in presentations in online learning to rank scenario after round $T=10000$. Error bars denote the standard deviation. TopRank hyperparameter $\delta$ was optimized in a held-out experiment.}
\label{fig:toprank}
\end{figure}

{\bf Learning from Pairwise Preferences}
We briefly focus on the case where $L=2$, of particular interest since this specific instantiation of the problem can be viewed as an instance of {\em dueling bandits} \citep{dueling}---where couples of options $C, |C| = 2$ are presented (a {\em duel} is set) and the ``winner'' $k \in C$ observed. 
Specifically, we compare Dirichlet-Luce Thompson sampling with a specialized dueling bandits algorithm, \textit{Double Thompson Sampling} (DTS) \citep{doublets}.
We simulate user feedback to pairwise presentations, first assuming a true $\theta^*$ and simulating choices as in $L$-wise presentations.
We also introduce a more challenging scenario---an irrational user whose preferences do not admit a total ordering.
In this case, we assume a {\em Condorcet} winner, an option preferred over all others, where preferences are otherwise cyclical and simulated from a pairwise preference probabilities table $p(k | \{k, l\})$.

\begin{figure*}[h!]
\centering
\begin{subfigure}{0.32\textwidth}
	\includegraphics[width=\linewidth]{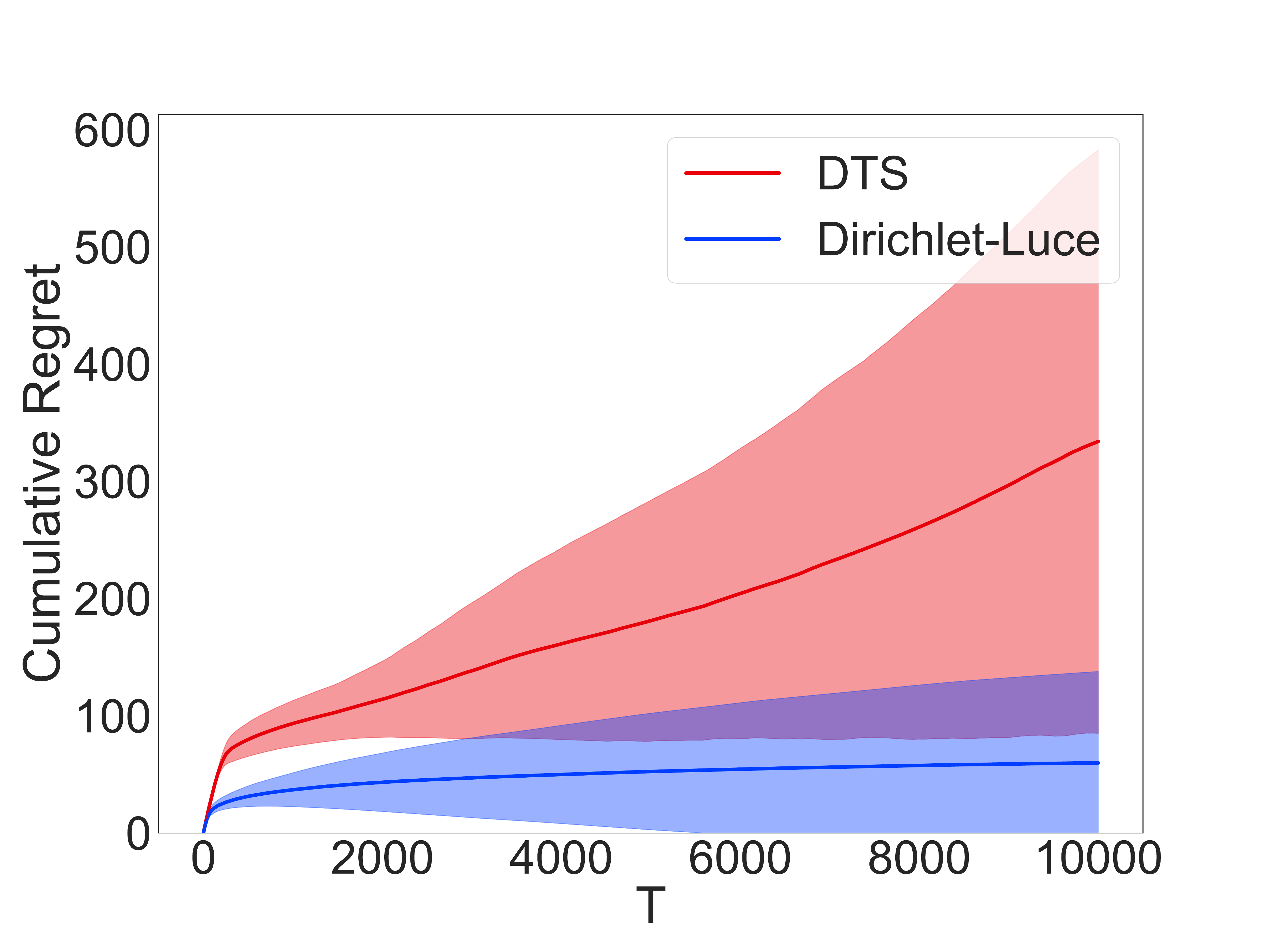}
	\caption{Sparse $\theta^*$}
\label{fig:transitive1}
\end{subfigure}
\begin{subfigure}{0.32\textwidth}
	\includegraphics[width=\linewidth]{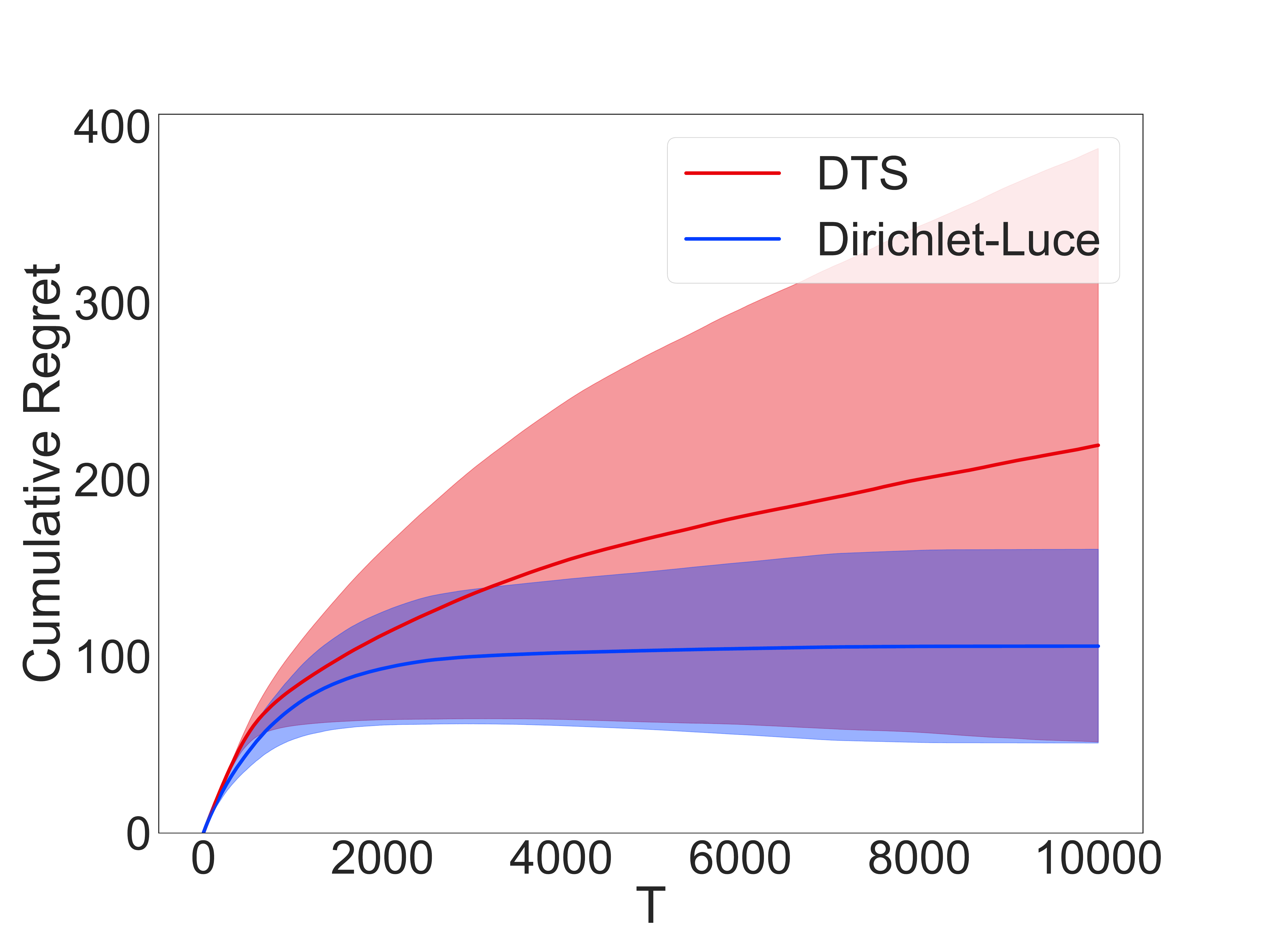}
	\caption{Dense $\theta^*$}
\label{fig:transitive2}
\end{subfigure}
\begin{subfigure}{0.32\textwidth}
	\includegraphics[width=\linewidth]{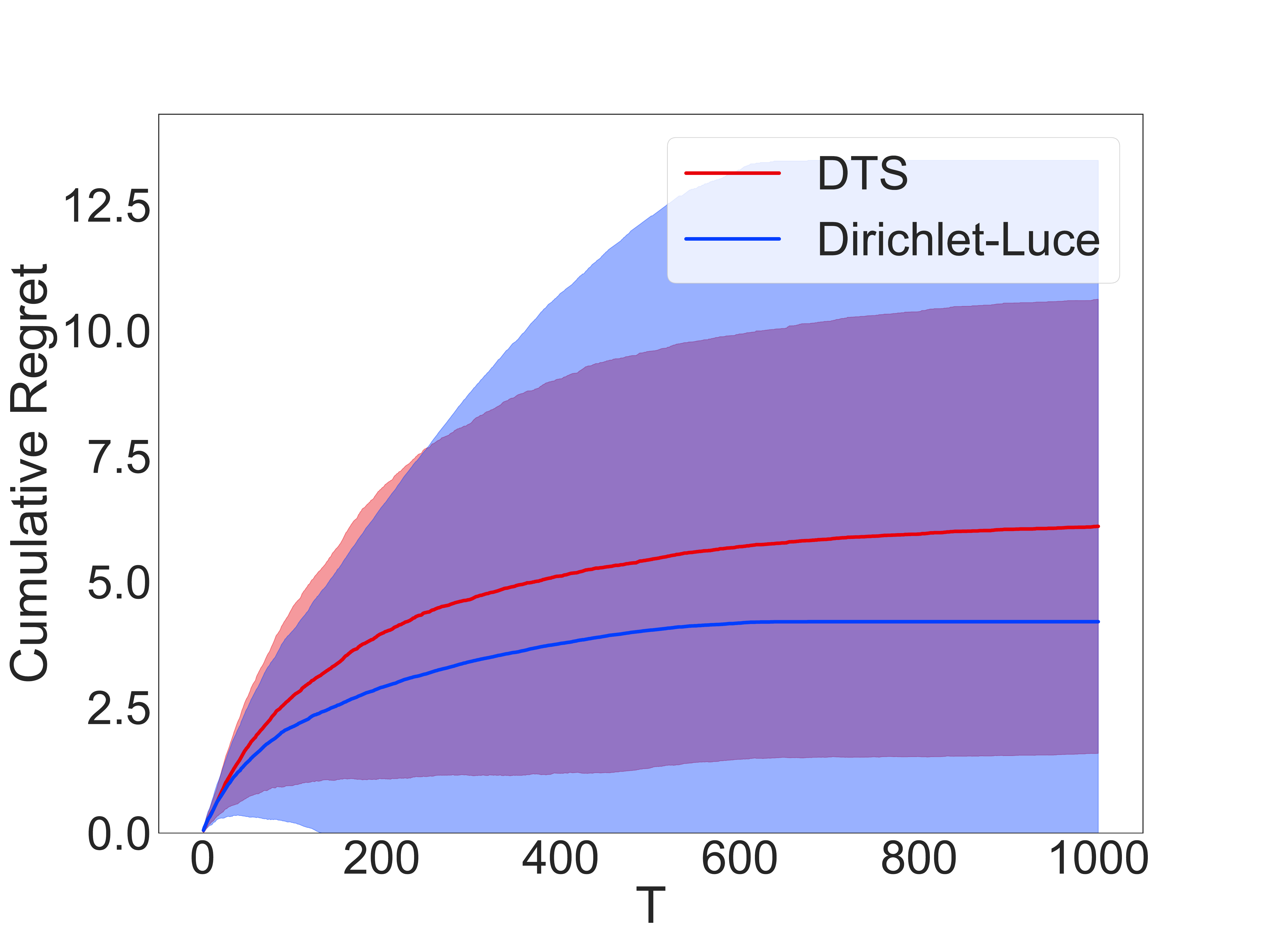}
	\caption{Cyclic}
	\label{fig:cyclicdueling}
\end{subfigure}
\caption{Average cumulative weak dueling regret (lower is better) in dueling bandits setup for Double Thompson Sampling (DTS) and the proposed presentation mechanism (JTS) with simulated transitive (a and b) or cyclic (c) preference feedback. Shaded regions denote the standard deviation.}
\label{fig:transitive}
\end{figure*}

We report a comparison of {\em weak} dueling regret \citep{dueling} in Figure~\ref{fig:transitive}.
We find that the proposed algorithm results in substantially lower regret when stochastic transitivity is assumed to hold. 
Furthermore, despite the fact that the DTS algorithm's model assumption can capture non-transitive (cyclical) choice behavior, our algorithm outperforms DTS in the average case. 
That said, however, our algorithm results in high variance of regret.
We provide further details of our experiment setup in the supplementary material.

{\bf Growing Presentation Sizes}
Our presentation mechanism outperforms baselines in both pairwise ($L=2$) and subset-wise ($L=5$) selection tasks.
In Figure~\ref{fig:battling-regret}, fixing $N=2$, we explore how learning speed improves as the system is allowed to make larger presentations. 
As expected, growing presentation sizes leads to lower cumulative regret, \ie, the algorithm learns to present the top 2 options sooner.
We also report the dimensionality of the statistic $\mu$---the number of unique presentations explored before converging to a preference estimate.
Despite the potentially high complexity, the mechanism maintains manageably low-dimensional statistics (Figure~\ref{fig:battling-card}). 

\begin{figure}[h!]
\centering
\begin{subfigure}{0.42\textwidth}
	\includegraphics[width=\linewidth]{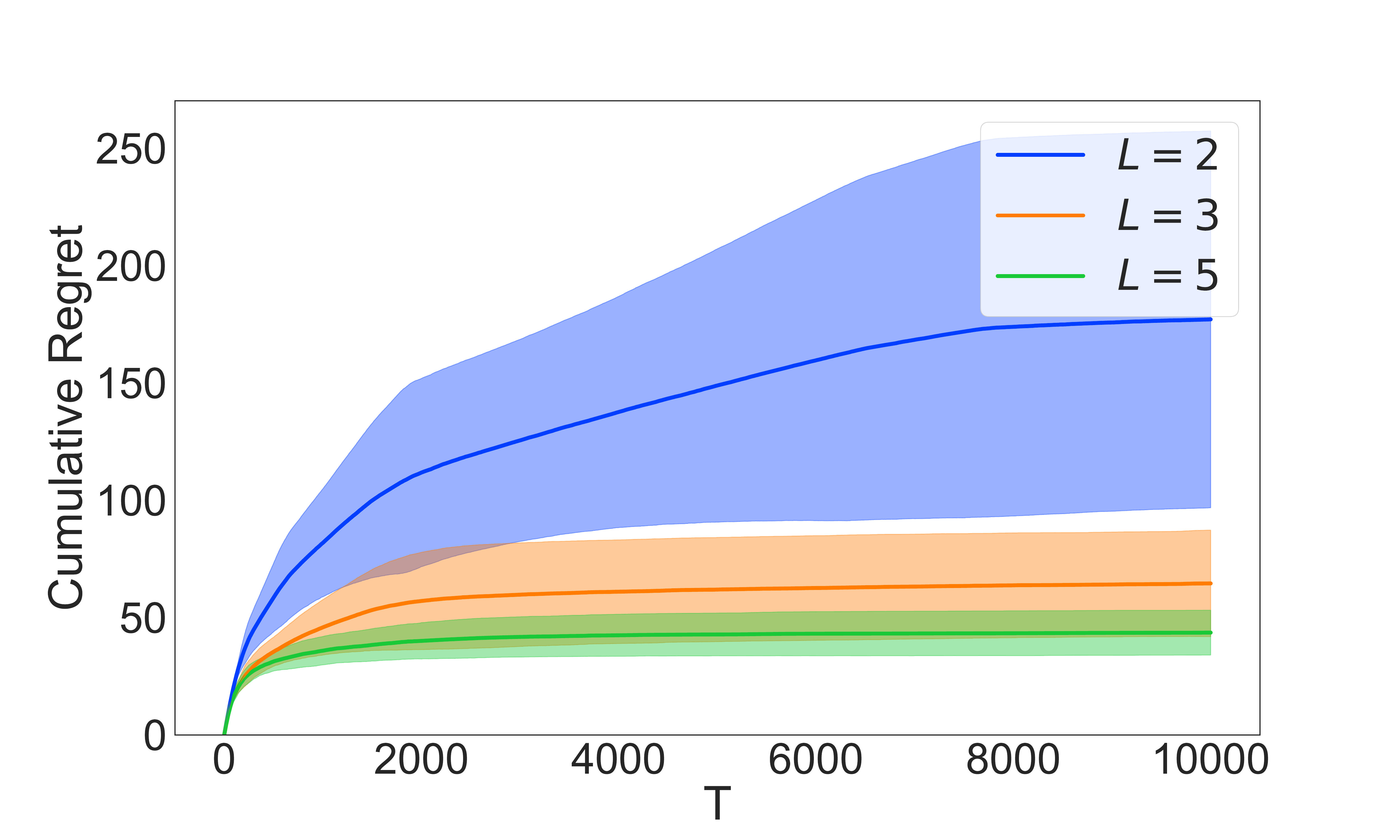}
	\caption{Performance at top-2}
	\label{fig:battling-regret}
\end{subfigure}
\begin{subfigure}{0.42\textwidth}
	\includegraphics[width=\linewidth]{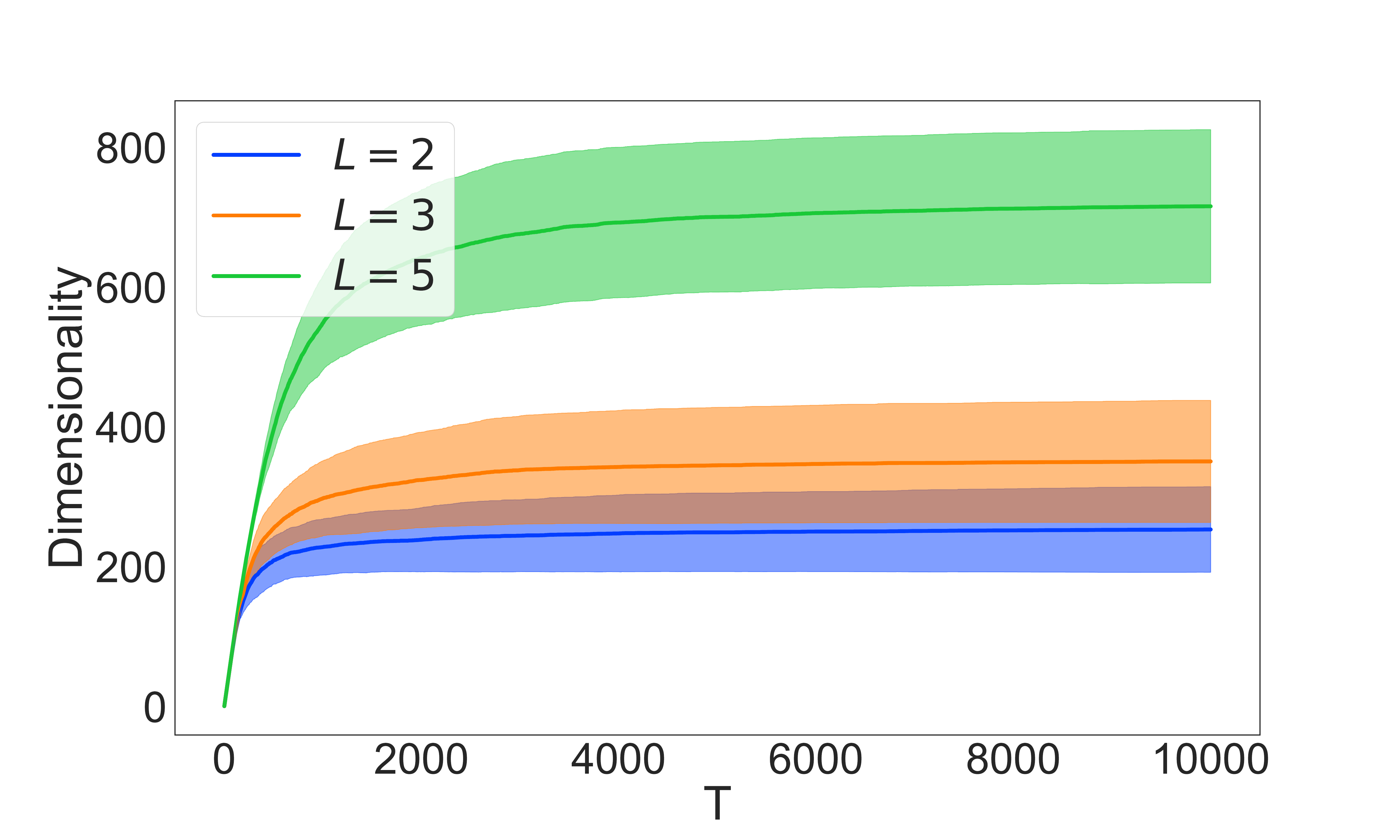}
	\caption{Number of unique presentations}
	\label{fig:battling-card}
\end{subfigure}
\caption{Average cumulative regret (over 10 runs) at top-2 options (out of 100) included in presentations with different sizes (a), and the effective dimensionality of the statistic $\mu$ (b) over the course of interactions. Shaded regions denote the standard deviation.}
\label{fig:battling}
\end{figure}

%% file: discussion.tex
\section{Discussion}\label{sec:discussion}
In this paper, we studied online preference elicitation in an interactive system where the user chooses one among a limited number of systematically presented options. 
We introduced Dirichlet-Luce, a Bayesian choice model that is aware of limited exposure and admits efficient learning and inference.
Bayesian treatment of the model paved the way for online estimation of user preferences through a novel presentation mechanism based on Thompson sampling. 

Overall, our work combines elements from Bayesian inference, choice modelling, active learning to rank, and bandit algorithms literatures to propose a novel framework to tackle self-reinforcing feedback loops in interactive personalization systems. 

Many exciting avenues for future work remain to be explored.
Our model can be used as a building block in more complex model architectures. 
The proposed approach can be adapted to a multi-user setting as in collaborative filtering where users are assumed to share interests.

%% file: supplement_model.tex
\section{Supplementary Material}\label{appx:model}
\subsection{Dirichlet-Luce Model}
\subsubsection{Posterior Predictive Inference and Partition Function}

The Dirichlet-Luce posterior log potential $\phi(\theta)$ was defined
\begin{equation}\label{eq:potential}
\phi(\theta) 
= \log \left[\prod_k  \theta_k^{y_k+\alpha_k-1} \prod_{C\in \mathcal{C}} \left(\sum_{\kappa \in C}\theta_\kappa\right)^{-\beta(C)-\mu(C)} \right].
\end{equation}
Here, as in the main text, the index $C$ runs over L-way combinations of the set of choices $[K]$, and $k$ indexes the choices themselves.
For ease of exposition, let us introduce the indicator matrix $Z \in \{0, 1\}^{K \times {K \choose L}}$, defined $z_{k, C} = [k \in C]$, 
\begin{equation}\label{eq:pot_2}
\phi(\theta) 
= \log \left[ \prod_k \theta_k^{y_k+\alpha_k-1} \prod_{C\in \mathcal{C}} (\mathbf{z}_{:, C}^\top \theta) ^{-\beta(C)-\mu(C)} \right].
\end{equation}

Many quantities of interest for the Bayesian choice model are written in terms of the normalizing constant (partition function) of $\phi$, $\int_\Delta \exp \phi(\theta) d\theta$. 
The normalizer of $\phi$ is written in terms of a special hypergeometric function known as Carlson's $\mathcal{R}$ function,
\begin{align*}  \label{eq:R}
\mathcal{R}(\alpha + \mathbf{y}, Z, \beta + \mu) &= \dfrac{1}{B(\alpha + \mathbf{y})} \int_\Delta \exp \phi(\theta) d\theta \\
&= \dfrac{1}{B(\alpha + \mathbf{y})} 
   \int_\Delta \prod_k  \theta_k^{y_k+\alpha_k-1} \prod_{C\in \mathcal{C}} \left(\sum_{\kappa \in C}\theta_\kappa\right)^{-\beta(C)-\mu(C)} d\theta,
\end{align*}
where $B(\alpha)$ denotes the multivariate Beta function.

We can write $p(k_{1:T}\mid C_{1:T},\alpha, \beta)$, probability of a sequence of choices conditioned on a sequence of presentations and hyperparameters, as ratios of $\mathcal{R}$ functions

\begin{align*}
p(k_{1:T}\mid C_{1:T},\alpha, \beta) &= \int_{\Delta} p(\theta\mid \alpha, \beta) p(k_{1:T}\mid C_{1:T}, \theta)  d\theta \\
&= \int_{\Delta} p(\theta\mid \alpha, \beta) \prod_k \theta_k^{y_k} \prod_{C\in \mathcal{C}}{(\theta^\top \mathbf{z}_{:, C}})^{-{\mu(C)}} d\theta \\
&= \frac{\prod_k (\alpha_k)_{(y_k)}}{(\sum_k \alpha_k)_{(N)}} \frac{\mathcal{R}(\alpha+\mathbf{y}, Z, \beta+\mu)}{\mathcal{R}(\alpha, Z, \beta)},
\end{align*}
where the notation $(x)_{(n)} = \frac{\Gamma(x+n)}{\Gamma(x)}$ denotes the rising factorial.

This gives the predictive preference of option $k$ (over all $[K]$) as
\begin{equation*}
p(k|[K], \alpha, \beta)= \int_{\Delta} p(\theta\mid \alpha, \beta) \theta_k  d\theta = \frac{\alpha_k}{\sum_j \alpha_j} \frac{\mathcal{R}(\alpha + \mathbf{k}, Z, \beta + [0, 0, \cdots, 1])}{\mathcal{R}(\alpha, Z, \beta)},
\end{equation*}
where $\mathbf{k}$ is the indicator vector of size $K$ where all but the $k$-th element are 0.

Computation of the $\mathcal{R}$ function was studied in \citep{rcomputation}. 
Here, we reiterate some key results to highlight implications for our case.

\begin{lemma}(\citep{rcomputation} Lemma 3.2)\label{lemma:r_zero}
Let $G \in \IR^{K \times C}$, $\mathbf{a} \in \IR^{K}, \mathbf{b} \in \IR^{C}$.
Assume $\mathbf{b}$ admits a permutation such that $\mathbf{b} = [\Tilde{\mathbf{b}} \,\, \mathbf{0}]$ where $\mathbf{0}$ denotes a vector of zeros. Also, let $\Tilde{G}$ denote $G$ with columns permuted conformably to $\mathbf{b}$. Then, $\mathcal{R}(\mathbf{a}, G, \mathbf{b}) = \mathcal{R}(\mathbf{a}, \Tilde{G}, \Tilde{\mathbf{b}}).$
\end{lemma}

\begin{lemma}(\citep{rcomputation})\label{lemma:r_commute}
Let $G \in \IR^{K \times C}$, $\mathbf{a} \in \IR^{K}, \mathbf{b} \in \IR^{C}$, and $\sum_i a_i = \sum_j b_j$. Then, $\mathcal{R}(\mathbf{a}, G, \mathbf{b}) = \mathcal{R}(\mathbf{b}, G^\top, \mathbf{a}).$
\end{lemma}

These results yield an important implication for approximating $\mathcal{R}$ functions.
The dimension of the required integration depends on the sparseness of sufficient statistics, \ie, scales proportionally to the number of unique presentations in the data set.
Also, in our case, the assumption of the second lemma holds by definition.
Then, the dimension of the integration can be simplified by relying on the commutativity implied in Lemma 2, and invoking the sparseness of sufficient statistics.

The $\mathcal{R}$ function admits efficient computation in some specific conditions, such as if the subsets of $[K]$ implied by the columns of $Z$ imply a hierarchy of set partitions on $[K]$.
In the general case, \citep{rcomputation} propose Laplace approximation around the MAP estimate $\theta^* = \argmax_{\theta} \phi(\theta).$
Assume $\lVert z_{:,C}\rVert_0 < L, \forall C$ and let $T$ denote the number of unique presentations in the data.
Laplace approximation requires computing the Hessian and taking its determinant in $O(K^3 + L^2T)$ time. However, when $K$ is reasonably small, the approximation is straightforward to implement.

\subsubsection{Maximum a Posteriori Estimation}

The gradient of the log posterior potential (\ref{eq:pot_2}) can be computed in $O(LT + K)$ time, as the sparsity of $Z$ can be invoked. 
In large data scenarios, the MAP estimate can be used as a point estimate of preferences as well as for approximate inference as described in the previous section.
However, the concern in Section~\ref{sec:dirichlet_luce} for the number of samples required for an accurate estimate should be raised for MAP estimates as well.

In Figure~\ref{fig:mle_norms}, we continue the synthetic experiment setup in Section~\ref{sec:dirichlet_luce}. 
Namely, we generate pairwise comparisons from MergeRank for random dense $\theta^*$. 
We then increase the number of random samples $T$ taken from this data and study the accuracy of MAP preference estimates attained by our model. 
We report the error in estimated $\theta$ in terms of total absolute deviation and Kullback-Leibler divergence. 
We find that the MAP routine yields an accurate preference estimate in a reasonable number of observations.

\begin{figure}[t]
\centering
	\includegraphics[width=\textwidth]{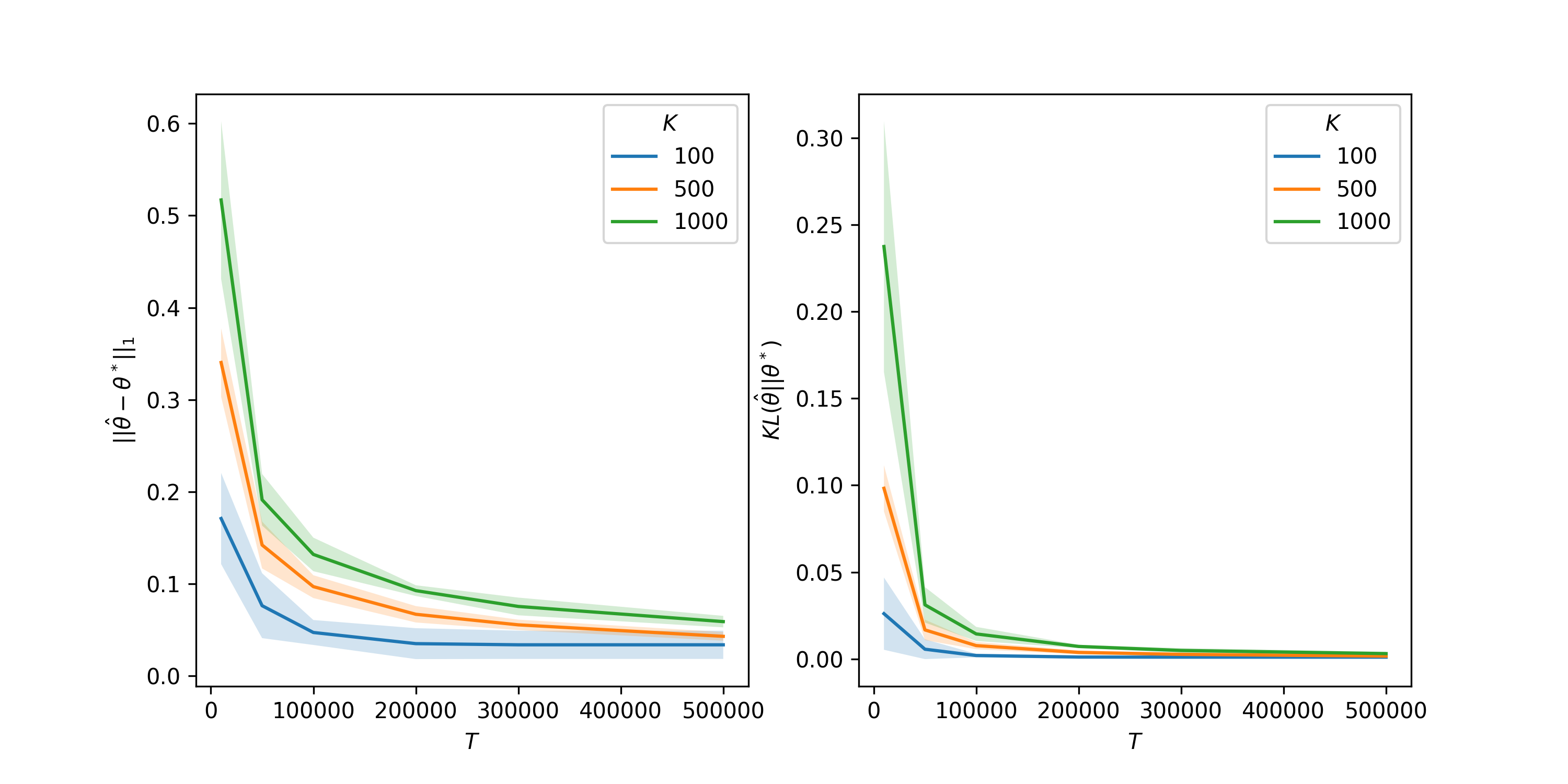}
	\caption{Total absolute deviation and Kullback-Leibler divergence of the estimated $\hat{\theta}$ from the optimal $\theta^*$. }
\label{fig:mle_norms}
\end{figure}

\subsubsection{Independence of Unexplored Options}\label{apx:fairness}

In Section~\ref{sec:dirichlet_luce}, we stated that the posterior leads to {\it fair} preference estimates since choice probabilities of unexplored ($k \in C$ implies $\mu(C)=0$) options are invariant independent of other choices. 
Here, we demonstrate this result deriving the marginal distribution of posterior choice probabilities for such options. 

\begin{lemma} (Independence of unexplored options)
Assume $\mu(C) = 0, \forall C \ni \ell$.
It then follows, $p(\theta_\ell\mid \alpha, \beta_0, k_{1:T}, C_{1:T}) = p(\theta_\ell\mid \alpha, \beta_0)$.
\end{lemma}
\begin{proof}
Assume, without loss of generality, option $1$ was never presented. 
Since with $\beta = \beta_0$, prior $\theta$ is Dirichlet distributed with parameter $\alpha$, prior marginal $\theta_1|\alpha, \beta_0$ is Beta distributed with parameters $(\alpha_1, \sum_{j=2}^K \alpha_j)$. 
The posterior marginal is
\begin{equation*}
p(\theta_1\mid \alpha, \beta_0, k_{1:T}, C_{1:T}) \propto \int_{T_1} \prod_{j=1}^{K} \theta_{j}^{\alpha_j - 1} \prod_{C\in\mathcal{C}} \frac{\prod_{j=2}^{K}\theta_{j}^{\nu(j,C)}}{(\sum_{i\in C} \theta_i)^{\mu(C)}} d\theta_{2:K-1}
\end{equation*}
where $T_1 = \{(\theta_2, \cdots, \theta_K)\mid \sum_{j=2}^K\theta_j=1-\theta_1, \theta_j>0\}$.

With a change of variables $u_j = \frac{\theta_j}{1-\theta_1}$ for $j \in \{2, \cdots, K\}$, we obtain:
\begin{align*}
p(\theta_1\mid \alpha, \beta_0, k_{1:T}, C_{1:T})
&\propto \theta_1^{\alpha_1 -1} (1-\theta_1)^{\sum_{j=2}^K \alpha_j - 1} \int_{\Delta} \prod_{j=2}^{K} u_{j}^{\alpha_j - 1} \prod_{C\in\mathcal{C}} \frac{\prod_{j=2}^{K}\left((1-\theta_1)u_j\right)^{\nu_(j,C)}}{\left((1-\theta_1)(\sum_{i\in C} u_i)\right)^{\mu(C)}} du\\
&= \theta_1^{\alpha_1 -1} (1-\theta_1)^{\sum_{j=2}^K \alpha_j - 1} \int_{\Delta} \prod_{j=2}^{K} u_{j}^{\alpha_j - 1} \prod_{C\in\mathcal{C}} \frac{\prod_{j=2}^{K}u_j^{\nu(j,C)}}{\left(\sum_{i\in C} u_i\right)^{\mu(C)}} du\\
&\propto \theta_1^{\alpha_1 -1} (1-\theta_1)^{\sum_{j=2}^K \alpha_j - 1}
\end{align*}
Then, the posterior $\theta_1\mid \alpha, k_{1:T}, C_{1:T} \sim \mathcal{B}(\alpha_1, \sum_{j=2}^K \alpha_j)$ is also Beta distributed with parameters $(\alpha_1, \sum_{j=2}^K \alpha_j)$, identically to the prior $p(\theta_1|\alpha, \beta_0)$.
\end{proof}

%% file: supplement_ts.tex
\subsection{Learning to Present: Illustrations}
In Dirichlet-Luce model, {\it never-presented} options are never penalized, and $K$ may grow arbitrarily. Figure~\ref{fig:supp_ts2} illustrates this, underscoring the contrast with Dirichlet-Multonimal. 
Particularly, we assume an interaction scenario where inferior options are initially presented several times. 
Independence of unexplored alternatives ensures that other, originally superior options will be considered by the system to be presented. As more evidence is collected based on user feedback, they will eventually dominate the presentations.

Randomization is essential in a presentation strategy.
Due to transitivity, posterior estimates for scarcely presented options might still be biased. That is, once presented, the choice probability of an option is no longer independent of other choices made.

We can illustrate with an example, that Thompson sampling fixes the potential bias towards a scarcely presented option. 
Let us assume there are $K=5$ options, and $\theta^*_i>\theta^*_j$ whenever $i<j$. 
Further, we assume an interaction scenario that the user was presented $\{1,5\}$ ten times and $1$ was always preferred, and then the option $5$ was preferred to $2$.
The dependencies due to transitivity would penalize the posterior choice probability of the option $2$ as the options other than $2$ are preferred to $5$ or the option $1$ is preferred to the others. 
This bias introduced from the transitivity is remedied by randomization by Thompson sampling, as illustrated in Figure~\ref{fig:supp_ts1}.

\begin{figure*}[h!]
\centering
\begin{subfigure}{0.9\textwidth}
    \includegraphics[width=\linewidth]{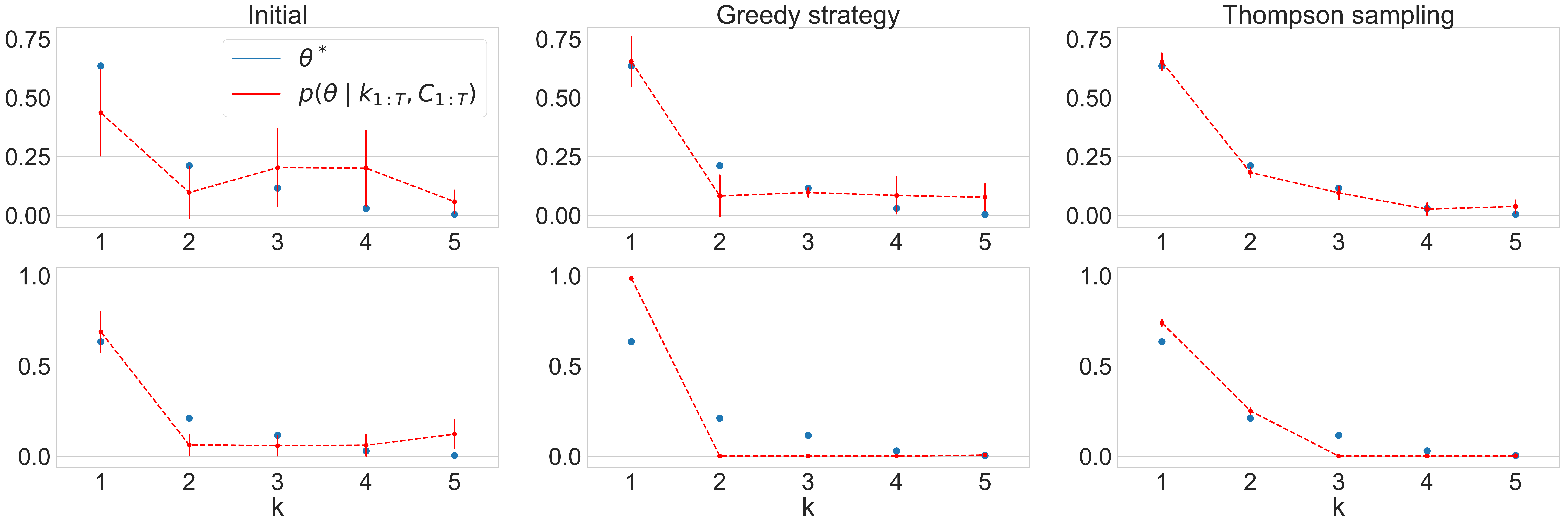}
    \caption{When an inferior option 5 was preferred to option $2$, Thompson sampling fixes the potential bias caused by transitivity.}
    \label{fig:supp_ts1}
\end{subfigure}
\begin{subfigure}{0.9\textwidth}
    \includegraphics[width=\linewidth]{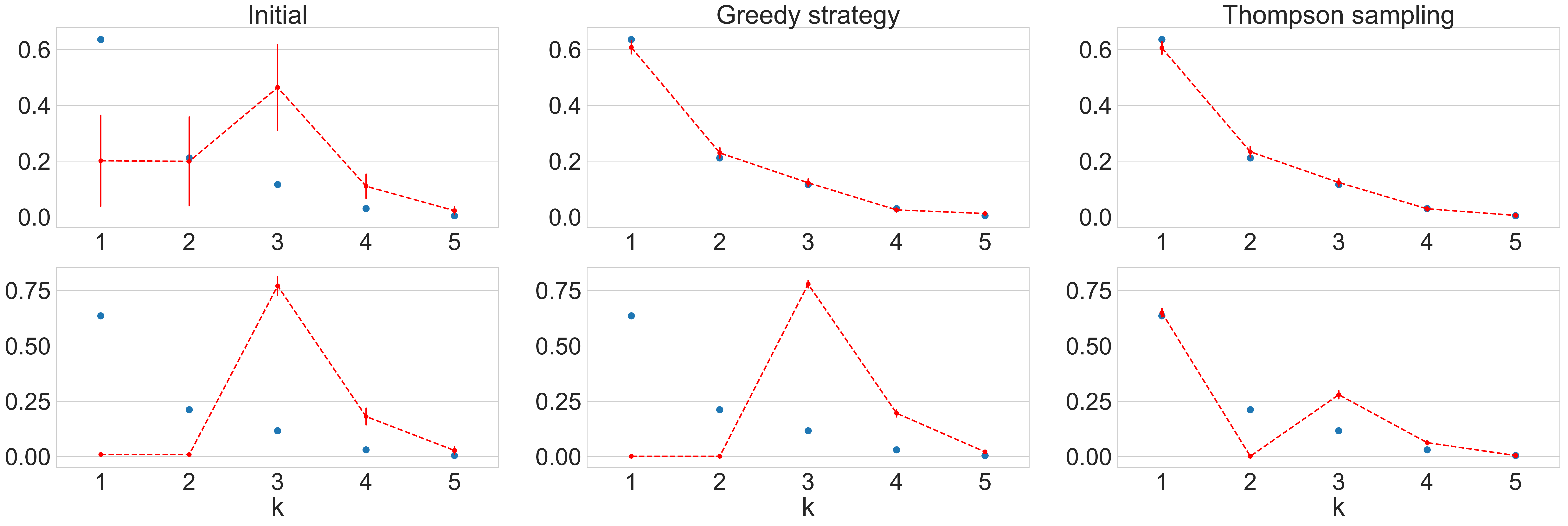}
    \caption{Dirichlet-Luce model (first row) is resilient to initial conditions such as when three inferior options (here the options $3$, $4$, and $5$, and feedback is simulated) were initially promoted (presented 100 times), where Dirichlet-Multinomial (second row) fails.}
    \label{fig:supp_ts2}
\end{subfigure}
\caption{Posterior distribution of preferences under Dirichlet-Luce and Dirichlet-Multinomial models with greedy or Thompson sampling based presentations in two scenarios.}
\label{fig:supp_ts}
\end{figure*}

%% file: supplement_smc.tex
\subsection{Sequential Sampling Procedure}
We detail the sequential sampling procedure that was described in Section~\ref{sec:presentation}. The procedure is designed for interactive environments, following \citet{chopin2002}.

We start with a flat prior on the choice probabilities of options. Therefore initial particles, indexed by $i$, are drawn from $\theta^{(i)} \sim \mathcal{D}(\mathbf{1})$ and assigned unit weights. At time $t$, we update particle weights $w^{(i)}$ based on choice $k_t$ restricted to presentation $C_t$.

\begin{equation*}
    w_t^{(i)}
        = w_{t-1}^{(i)} \dfrac
                { p_{t}(\theta^{(i)}) }
                { p_{t-1}(\theta^{(i)}) }
        = w_{t-1}^{(i)} \dfrac
                { \theta^{(i)}_{k_t} }
                { \sum_{\kappa \in C_t} \theta^{(i)}_{\kappa} }
\end{equation*}

Here, $p_t(\theta^{(i)})$ denotes the posterior density at time $t$ evaluated at the point $\theta^{(i)}$, short for $p_t(\theta^{(i)}\mid C_{1:t}, k_{1:t}, \alpha, \beta)$. We also keep track of the effective sample size, ESS, defined as follows:

\begin{equation*}
    \text{ESS}_{t} = \frac { (\sum_i w_t^{(i)})^2 } { \sum_i w_t^{(i)} }
\end{equation*}

Whenever $\text{ESS}_t$ drops below a certain threshold, we perform {\it multinomial resampling} followed by a {\it move} step with a {\it Metropolis-within-Gibbs} transition kernel which targets the posterior at $t$, $p_t(\theta)$. In order to apply any Gibbs sampling \citep{gelfand1990} based method, we need the full-conditional densities, $f_j(\theta_j \mid \theta_{-j})$, $1 \leq j \leq K-1$, where $\theta_{-j}$ denotes all coordinates except $j$. In general, full conditional densities do not have a standard form. Therefore we can't sample from $f_j$ directly. There are various alternatives for sampling from an unnormalized univariate density, e.g slice sampling \citep{neal2003}, ARMS \citep{gilks1995}. All of these methods require multiple evaluations of $f_j$, which is prohibitive in our case since evaluation of $f_j$ at any given point has a considerably high computational cost. We use a simpler but effective Metropolis scheme which requires only two evaluations of $f_j$ for each coordinate, as follows:
The coordinates except $j$ constrain the $j$'th coordinate to an interval, $\mathcal{I} = \left( 0, 1 - \sum_{j^{\prime} \neq j} \theta_{j^{\prime}} \right)$. We sample $\hat{\theta}_j \sim \mathcal{U}(\mathcal{I})$, and accept it with probability $\text{min}\{1, f_j(\hat{\theta}_j) / f_j(\theta_j) \}$. 

The complete presentation mechanism along with the sequential sampling routine is listed in Algorithm \ref{alg:sequential}. 

\begin{algorithm}[h!]
    \caption{Sequential Presentation Mechanism}
    \begin{algorithmic}\label{alg:sequential}
        \STATE {\bfseries Input:}
        \STATE $T:$ Number of interactions
        \STATE $N:$ Number of particles
        \STATE Initialize $\alpha$, and set $\beta\gets\beta_0$
        \STATE $\mu(C)\gets0$ for all $C\in\mathcal{C}$
        \STATE $y_k \gets 0$ for all $k\in [K]$
        \STATE $\theta^{(i)} \sim \mathcal{D}(\mathbf{1}), \quad i=1,2,\dots,N$
        \STATE $w_0^{(i)} = 1, \quad i=1,2,\dots,N$
        \FOR{$t=1$ {\bfseries to} $T$}
        \STATE $\theta^{(n)} \gets \theta^{(i)}$ with probability $w_t^{(i)} / \sum_{i^{\prime}} w_t^{(i^{\prime})}$
        \STATE Form $C_t$ with top $L$ elements of $\theta^{(n)}$ 
        \STATE Get preference feedback $k_t$ to $C_t$
        \STATE {\it // Update sufficient statistics}
        \STATE $\mu (C_t) \gets \mu (C_t) + 1$
        \STATE $y_{k_t} \gets y_{k_t} + 1$
        \STATE $w_t^{(i)} = w_{t-1}^{(i)}
                \dfrac { \theta^{(i)}_{k_t} }
                { \sum_{\kappa \in C_t} \theta^{(i)}_{\kappa} }$
        \STATE $\text{ESS}_{t} = (\sum_i w_t^{(i)})^2 / \sum_i w_t^{(i)}$
        \IF{$ESS_t < 0.5 * N$}
        \STATE {\it // Perform multinomial resampling}
        \FORALL{$j \in \{ 1,2,\dots,N \} $}
        \STATE $\theta^{(j)} \gets \theta^{(i)}$ with probability $w_t^{(i)} / \sum_{i^{\prime}} w_t^{(i^\prime)} $
        \ENDFOR
        \STATE {\it // Move particles}
        \FORALL{$i \in \{ 1,2,\dots,N \} $}
        \FORALL{$j \in \{ 1,2,\dots,K-1 \} $}
        \STATE $r = 1 - \sum_{j^{\prime} \neq j} \theta_{j^{\prime}}$
        \STATE $\hat{\theta}^{(i)}_j \sim \mathcal{U}(0, r)$
        \STATE $\lambda = \text{min}\{1, f_j(\hat{\theta}^{(i)}_j) / f_j(\theta^{(i)}_j) \}$
        \STATE $u \sim \mathcal{U}(0,1)$
        \IF{$u<\lambda$}
        \STATE $\theta^{(i)}_j = \hat{\theta}^{(i)}_j$
        \COMMENT{Accept}
        \ENDIF
        \ENDFOR
        \ENDFOR
        \ENDIF
        \ENDFOR
    \end{algorithmic}
\end{algorithm}

%% file: supplement_sim.tex
\subsection{Simulation Details}

\subsubsection{Dueling Bandits Simulations}
As we stated in Section~\ref{sec:simulations}, we assumed two kinds of feedback scenarios in the pairwise preferences case. 
In the first scenario, we assume that preferences are transitives, and a {\it Plackett-Luce} feedback model. 
We said that the simulations were set up based on a sparse, and then a dense $\theta^*$. 
Figure~\ref{fig:suppts} show these 50-dimensional vectors.

In the cyclic preferences scenario, we used the following pairwise preferences matrix \citep[see][Table 1-c]{rmed}:
\[
\begin{bmatrix}
    0.5 & 0.6 & 0.6 & 0.6 \\
    0.4 & 0.5 & 0.9 & 0.1 \\
    0.4 & 0.1 & 0.5 & 0.9 \\
    0.4 & 0.9 & 0.1 & 0.5 \\
\end{bmatrix}
\]
Say, $i\succ j$ means  that $p(i\mid\{i,j\}) > 0.5$. Option $1$ is the {\it Condorcet} winner, since $1 \succ j$ for all $j \neq 1$. Otherwise preferences are cyclic, as $2\succ 3$, $3 \succ 4$, but $4 \succ 2$.

\subsubsection{Larger Presentations}
For presentations including $L\geq2$ options, we simulated the interactions to demonstrate the regret at top-2 positions, and also reported the number of unique presentations made by the system over the course of 10000 interactions. 
For simulations we used a 100-dimensional sparse $\theta^*$, shown in Figure~\ref{fig:suppbs}

Finally for online learning to rank experiments with presentation size $L=5$, we used the $\theta^*$ in Figure~\ref{fig:suppranking}
\begin{figure*}[h]
	\begin{subfigure}{\textwidth}
		\includegraphics[width=\linewidth]{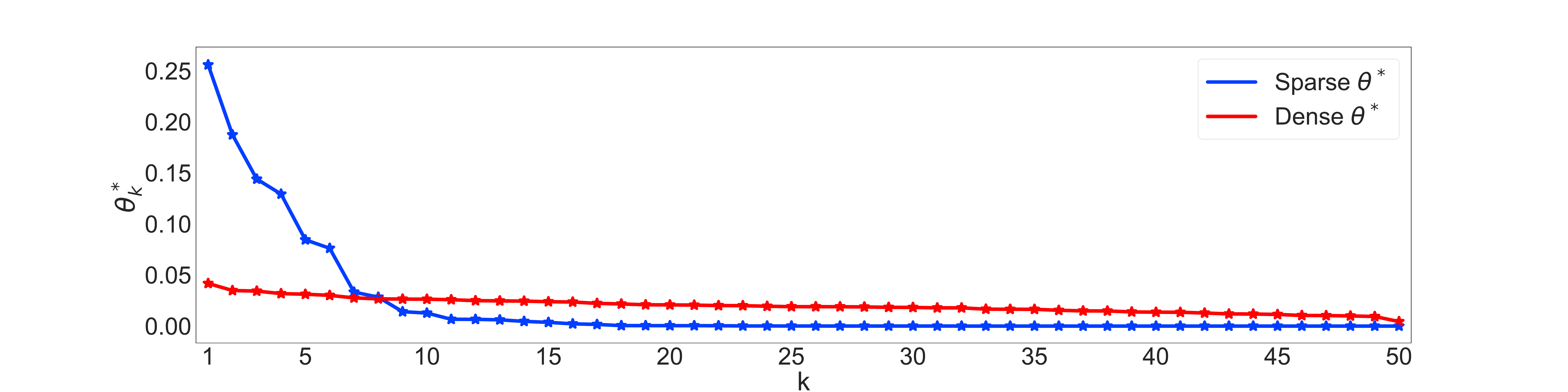}
		\caption{$\theta^*$ in Dueling Bandits Experiments}
		\label{fig:suppts}
	\end{subfigure}
	\begin{subfigure}{\textwidth}
		\includegraphics[width=\linewidth]{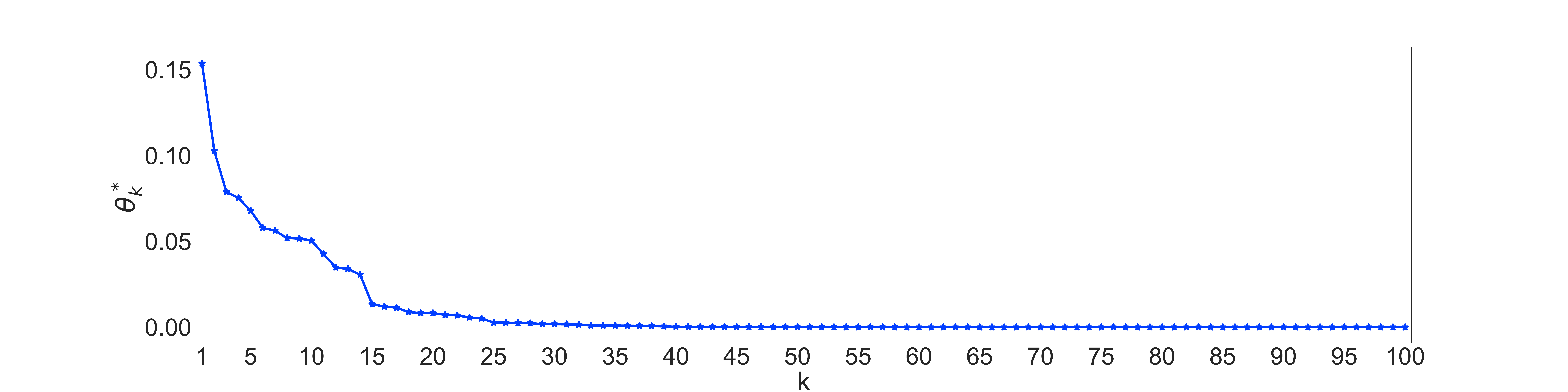}
		\caption{$\theta^*$ (100-dimensional) in Top-2 Performance Experiments}
		\label{fig:suppbs}
	\end{subfigure}
		\begin{subfigure}{\textwidth}
		\includegraphics[width=\linewidth]{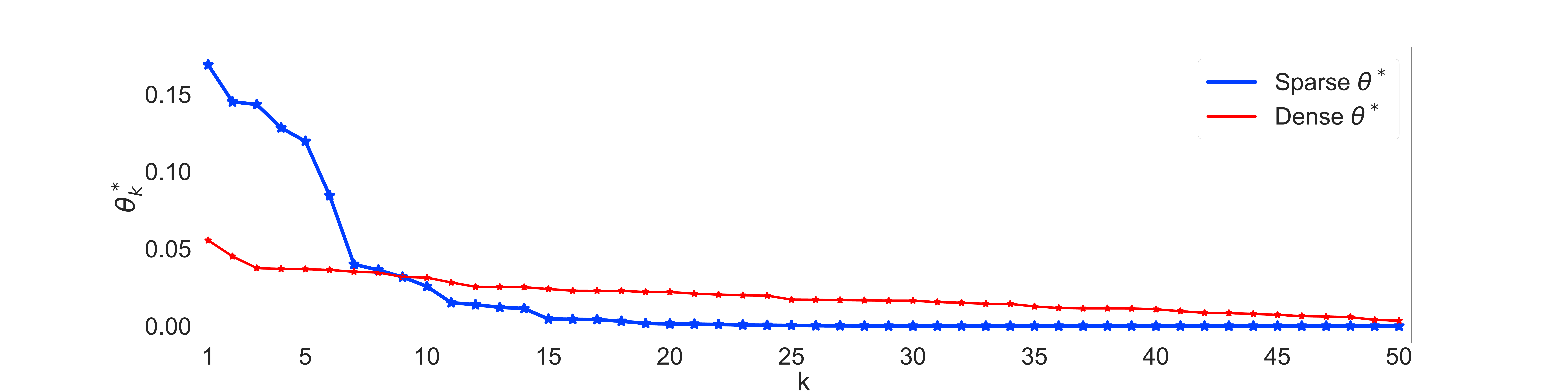}
		\caption{$\theta^*$ in Online Learning to Rank Experiments}
		\label{fig:suppranking}
	\end{subfigure}

	\caption{Simulated $\theta^*$'s, sorted in descending order for visualization}
\end{figure*}